%% file: main.tex
\documentclass{article}

\usepackage[margin=1in]{geometry}

\usepackage[colorlinks=true, citecolor=blue]{hyperref}       
\usepackage{url}            
\usepackage{booktabs}       
\usepackage{amsfonts}       
\usepackage{nicefrac}       
\usepackage{xcolor}         
\usepackage{wrapfig}
\usepackage[font=small]{caption, subcaption}

\usepackage{mathnotation}

\newcommand{\lin}[1]{\textcolor{red}{[Lin: #1]}}

\newcommand{\bruce}[1]{\textcolor{cyan}{[Bruce: #1]}}

\title{On the optimal regret of collaborative personalized linear bandits}

\author{
	Bruce Huang \\
	University of California, Los Angeles\\
	\texttt{brucehuang@ucla.edu} \\
	\and
	Ruida Zhou \\
	University of California, Los Angeles\\
	\texttt{ruida@ucla.edu} \\
	\and
	Lin F. Yang \\
	University of California, Los Angeles\\
	\texttt{linyang@ee.ucla.edu} \\
	\and
	Suhas Diggavi \\
	University of California, Los Angeles\\
	\texttt{suhas@ee.ucla.edu} \\
}

\begin{document}

\date{}
\maketitle

\input{Sections/abstract}
\input{Sections/introduction}
\input{Sections/formulation}
\input{Sections/results}

\input{Sections/lower_bound}

\input{Sections/alg_regret}

\input{Sections/experiments}
\input{Sections/conclusion}

\newpage
\bibliographystyle{plain}
\bibliography{refs}

\newpage
\appendix

\end{document}

%% file: Sections/abstract.tex
\begin{abstract}
Stochastic linear bandits are a fundamental model for sequential decision making, where an agent selects a vector-valued action and receives a noisy reward with expected value given by an unknown linear function. Although well studied in the single-agent setting, many real-world scenarios involve multiple agents solving heterogeneous bandit problems, each with a different unknown parameter. Applying single agent algorithms independently ignores cross-agent similarity and 
learning opportunities.
This paper investigates the optimal regret achievable in collaborative personalized linear bandits. 
We provide an information-theoretic lower bound that characterizes how the number of agents, the interaction rounds, and the degree of heterogeneity jointly affect regret. We then propose a new two-stage collaborative algorithm that achieves the optimal regret. Our analysis models heterogeneity via a hierarchical Bayesian framework and introduces a novel information-theoretic technique for bounding regret. Our results offer a complete characterization of when and how collaboration helps with a optimal regret bound $\tilde{O}(d\sqrt{mn})$, $\tilde{O}(dm^{1-\gamma}\sqrt{n})$, $\tilde{O}(dm\sqrt{n})$ for the number of rounds $n$ in the range of $(0, \frac{d}{m \sigma^2})$, $[\frac{d}{m^{2\gamma} \sigma^2}, \frac{d}{\sigma^2}]$ and $(\frac{d}{\sigma^2}, \infty)$ respectively, where $\sigma$ measures the level of heterogeneity, $m$ is the number of agents, and $\gamma\in[0, 1/2]$ is an absolute constant. In contrast, agents without collaboration achieve a regret bound $O(dm\sqrt{n})$ at best.

\end{abstract}

%% file: Sections/introduction.tex
\section{Introduction}

Stochastic linear bandits are a fundamental model in sequential decision-making, where in each round, an agent selects a vector-valued action and receives a stochastic reward with expectation equal to the inner product between the action and an unknown parameter vector~\cite{lattimore2020bandit}. This setting has been extensively studied in the single-agent case, where the goal is to maximize cumulative reward over time through efficient exploration and exploitation \cite{robbins1952some, lai1985asymptotically, auer2002using}.

In this paper, we consider a generalized multi-agent setting, motivated by many practical scenarios such as distributed decision systems \cite{landgren2016distributed, zhu2020distributed}, personalized adaptive interfaces \cite{liu2003adaptive, khamaj2024adapting}, and multi-device learning~\cite{cho2022flame}—where in each round, a set of $m$ agents must each select an action and receive feedback from a stochastic, potentially heterogeneous reward function.
A natural baseline approach is to run a standard single-agent linear bandit algorithm independently for each agent. While straightforward, this approach entirely ignores potential similarities across agents. When the agents' environments are similar, such independent learning leads to redundant exploration and a suboptimal use of data. In contrast, our goal is to develop collaborative learning algorithms that adaptively share information across agents—without knowing in advance how similar or different they are. This raises a fundamental question: 
{\textbf{How can collaboration reduce regret under unknown heterogeneity?}} 

To address the question, we model heterogeneity using a hierarchical/empirical Bayesian framework, where each agent's unknown parameter vector is drawn from an \emph{unknown} population distribution. Then the variance of this population captures the degree of heterogeneity. {Although the population distribution can be arbitrary in the general framework, in this work we specifically study the Gaussian population case. The Gaussian population distribution captures the heterogeneity through the population variance, around the \emph{unknown} arbitrary mean vector without any prior on it. Besides being a useful theoretical framework, effectively this population distribution has been extensively used in algorithms for personalized FL, in supervised frameworks \cite{ozkara2023statistical} and references therein.} {We consider an unknown population distribution and do not assume further prior distribution on the population parameters since the agents will not have such information in practice.}


We then propose a novel two-stage algorithm to leverage collaboration: in the first \textit{collaborative learning stage}, agents jointly perform phased elimination to identify a promising subset of actions by pooling their information. In the second \textit{personalized learning stage}, agents refine their individual models based on local feedback. This structure allows the algorithm to smoothly interpolate between joint learning (when agents are similar) and local adaptation (when they are not). We carefully design a threshold to guide the transition between the two stages, so that collaboration is activated only when it benefits learning. In the analysis, we show that the algorithm achieves a regret bound that can be significantly smaller than the independent-learning baseline when the number of rounds is small—highlighting the benefit of collaboration in the data-scarce regime.

To show the optimality of our algorithm, we derive an \textit{information-theoretic lower bound} on the regret across different levels of heterogeneity. A central challenge in this analysis stems from the nature of heterogeneity: unlike the classical single-agent setting, where the worst-case parameter governs the lower bound, we must now consider the \textit{worst-case population distribution} over agents’ parameters. To overcome this, we introduce novel techniques involving \textit{rotation and coupling} to construct hard instances and analyze the regret. Our lower bound precisely characterizes how the population variance $\sigma^2$ affects collaborative learning. Furthermore, we show that our proposed algorithm matches the lower bound, achieving regret that is \textit{both globally and individually optimal}.

\subsection*{Our Contributions}
\begin{itemize}
    \item \textbf{Regret Characterization.} We provide, to the best of our knowledge, the first complete characterization of the benefit of collaboration by deriving \textit{matching upper and lower bounds} on cumulative regret. For a fixed level of heterogeneity $\sigma$ and number of users $m$, we characterize the regret in three regimes by varying the number of rounds $n$. Specifically:
    \begin{itemize}
        \item When $n$ is small, we refer to as the \emph{data-scarce} regime where $n = o(d/m\sigma^2)$, collaboration yields regret $\tilde{O}(d\sqrt{mn})$, matching the optimal rate for a single-agent linear bandit with $mn$ rounds—showing that full data sharing is optimal.
        \item When $n$ is moderate, we refer to as the \emph{intermediate} regime where $n = \Theta(d/m^{2\gamma}\sigma^2)$ for any $\gamma \in [0, 1/2]$, the regret $\tilde{O}(dm^{1-\gamma}\sqrt{n})$ smoothly interpolates between the two extremes, with sharp dependence on the data size $n$ and heterogeneity level $\sigma$.
        \item When $n$ is large, we refer to as the \emph{data-rich} regime where $n = \omega(d/\sigma^2)$, regret grows as $\tilde{O}(dm\sqrt{n})$, recovering the performance of independent learning per agent.
    \end{itemize}

    \item \textbf{Information-Theoretic Lower Bound.} {To establish the main result, we develop a
    minimax regret lower bound for heterogeneous multi-agent linear bandits. Unlike standard settings that focus on worst-case individual parameters, our bound considers the worst-case population distribution and introduces new coupling and rotation techniques to handle this multi-agent heterogeneous setting.}

    \item \textbf{Optimal Algorithm.} We propose Collaborative Personalized Phased
    Elimination (CP-PE), a two-stage collaborative algorithm that adapts to the heterogeneity and \textit{achieves the minimax lower bound}. 
    Notably, the algorithm is \textit{individually optimal} for each agent and \textit{globally optimal} for cumulative performance.
\end{itemize}

\subsection{Related Work}

There is extensive work in the bandit and linear bandit literature (see, e.g., the comprehensive review in \cite{lattimore2020bandit}). Most of these works focus on the single-agent, non-distributed learning setting. Another line of research explores distributed learning \cite{hillel2013distributed, wangdistributed, huang2021federated, zhu2021federated}, but the majority of these studies consider the homogeneous case, where all agents collaboratively learn a shared parameter. Among studies on collaborative bandits with heterogeneous agents, several works consider clustered bandits algorithms~\cite{gentile2014online, landgren2016distributed, ghosh2021collaborative, liu2022federated,  blaser2024federated}. In most of the works, agents in the same cluster share a common unknown parameter. However, variations in users' parameters within the same cluster can significantly change the problem. Small variations between parameters in the same cluster are allowed in~\cite{blaser2024federated}. However, the variation is bounded by $O(1/m\sqrt{n})$, and applying single-agent methods to users with such variation gives optimal regret bound. This does not answer how heterogeneous data can help through collaboration.~\cite{ghosh2021collaborative} also propose a two-stage algorithm, while the algorithm does not adapt to heterogeneity. Other works characterize similarity between the agents with different structures like networks~\cite{buccapatnam2013multi, wu2016contextual}, or consider agents sharing a global component in their unknown parameters~\cite{li2022asynchronous}. Our work considers an unknown population model characterizing heterogeneity instead. We discuss the two most related works next.


After completing our work, we recently became aware of the work~\cite{do2023multi}. They consider heterogeneous linear contextual bandits, and they measure heterogeneity as the maximum difference, $\varepsilon$, between any pair of individual parameters. They propose an algorithm and provide upper and lower bounds that depend on $\varepsilon$ for the problem. In our work, we model heterogeneity differently through unbounded population distribution, and thus the value of $\varepsilon$ can go to infinity even when the level of heterogeneity is low in our case. 
Furthermore, their upper and lower bounds do not match in all regimes with up to a $\sqrt{d}$ gap for their model, whereas our results show matching bounds in all regimes. Additionally, due to the different modeling for heterogeneity, their lower bound techniques cannot be applied to our setup. We develop novel techniques in Section~\ref{sec:lower_bound} to resolve challenges emerging in our minimax lower bound.

Another related work is~\cite{hong2022hierarchical}, which proposes a hierarchical Bayesian framework to model heterogeneity in multi-armed bandits, 
and introduces a hierarchical Thompson Sampling algorithm for the setting. A major limitation of their approach is the assumption of a known prior distribution over the population model. The case of fixed unknown population parameters is important since agents do not have information on such prior distribution in practice. However, the case is not captured in their setting and their result does not apply. In contrast to their pure Bayesian setting, we adopt an empirical Bayes approach without assuming any prior distribution on the unknown population parameters, which can make the analysis significantly more challenging. Furthermore, their work does not provide a regret lower bound, whereas we establish a matching lower bound that demonstrates the optimality of our proposed algorithm. 

%% file: Sections/formulation.tex
\section{Problem Formulation}
We now formally present the problem definitions in the paper. 
We first recall the conventional stochastic linear bandits, let $n$ be the number of rounds. The user chooses an action $\vx_t$ from the action set $\mathcal{A} \subset \mathbb{R}^d$ in each round $t \in [n]$, and receive a stochastic reward $y_t = \vx_t^\top \vtheta + \epsilon_t$ where $\vtheta \in \mathbb{R}^d$ is an unknown parameter to the user and $\epsilon_t \sim \mathcal{N}(0, \sigma_0^2)$ is a Gaussian noise.

In our problem setting, there are $m$ users, each solving a stochastic linear bandit problem with a distinct parameter $\vtheta_i \in \mathbb{R}^d$ for $i = 1, 2, \ldots, m$. We model user heterogeneity using a hierarchical/empirical Bayesian framework~\cite{ozkara2023statistical}, where the parameters $\vtheta_i$ are sampled i.i.d. from a population-level Gaussian distribution:
\begin{align}
\label{eq:hier_model1}
\vtheta_i \sim \mathcal{N}(\vmu, \mat{C}),
\end{align}
with a covariance matrix $\mat{C}$ capturing heterogeneity across users, and an unknown mean $\vmu$. While we assume a Gaussian population distribution in the formulation, our algorithm and regret analysis extend to sub-Gaussian population distributions. See Section~\ref{app:subgaussian} for details.

To illustrate the role of heterogeneity, consider the special case where $\mat{C} = \sigma^2 \mat{I}$ for some $\sigma \geq 0$. In this case, a larger $\sigma$ corresponds to greater heterogeneity, while $\sigma = 0$ recovers the homogeneous setting in which $\vtheta_1 = \dots = \vtheta_m$. However, our results hold for general covariance matrices $\mat{C}$.

At time $t$, each user $i$ chooses an action $\vx_{i,t}$ from an action set $\mathcal{A}_i\subset \mathbb{R}^d$, and receives a reward $y_{i,t}$:
\begin{gather}
y_{i,t} = \vx_{i,t}^\top \vtheta_i + \epsilon_{i,t}, \quad \epsilon_{i,t} \sim \mathcal{N}(0, \sigma_0^2),
\end{gather}
where $\epsilon_{i,t}$ is an unknown ``local'' noise and $\sigma_0 > 0$ is the variance, assumed to be the same\footnote{Our results naturally generalize to heterogeneous local noise levels; we assume a uniform noise level $\sigma_0$ for clarity of presentation.} across users. A visualization of the model is provided in Figure~\ref{fig:hier}.
For the sake of presentation, we further assume that $\mat{C}$ is known and local noise has a unit variance,
$
\sigma_0^2 = 1.
$

\label{sec:formulation}
\begin{figure}[t]
    \centering
    \includegraphics[width=0.5\linewidth, trim={1 4cm 0 4cm},clip]{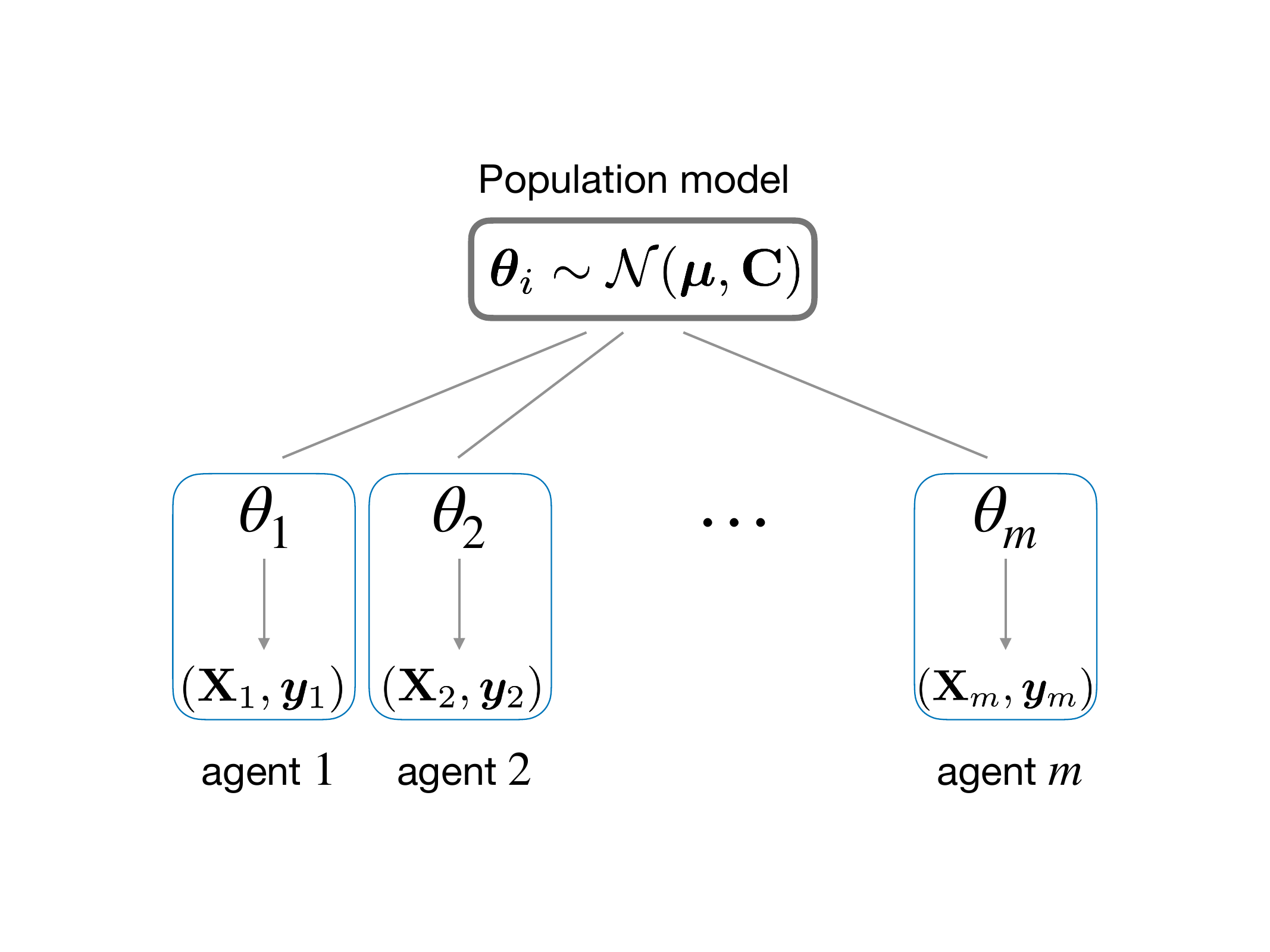}
    \caption{Hierarchical/empirical Bayesian framework}
    \label{fig:hier}
\end{figure}

We use the following regret for measuring the performance of an algorithm:

\vspace{0.1cm}
\begin{definition}[Joint pseudo-regret]
    For any $\vmu$, the joint pseudo-regret of the users is defined as
    \begin{align*}
        \hat{R}_{m,n}(\vmu) = \sum_{i=1}^m \sum_{t=1}^n \left( \langle \vx_i^*, \vtheta_i \rangle - \langle \vx_{i,t}, \vtheta_i \rangle \right)
    \end{align*}
where $\vx_i^* = \argmax_{\vx \in \mathcal{A}_i} \langle \vx, \vtheta_i \rangle$ is the optimal action for each user $i \in [m]$.
\end{definition}

\vspace{0.1cm}
\begin{definition}[Joint Bayesian regret]
    For any $\vmu$, the joint Bayesian regret of the users is defined as 
    \begin{align*}
        R_{m,n}(\vmu) = \mathbb{E} \left[ \sum_{i=1}^m \sum_{t=1}^n \left( \langle \vx_i^*, \vtheta_i \rangle - \langle \vx_{i,t}, \vtheta_i \rangle \right) \right]
    \end{align*}
    where $n>0$ is the number of total rounds of interactions. 
\end{definition}



\paragraph{Notations.}In the rest of the paper, we represent vectors in bold lowercase letters, such as $\vmu, \vtheta$. Bold capitalized letters $\mat{X}, \mat{C}$ are matrices, and $\mat{I}$ is the identity matrix. We denote $[m] = \{1,\dots,m\}$. For a real symmetric matrix $\mat{A}$, we define $\mat{A}^\dagger$ to be the pseudo-inverse of $\mat{A}$, obtained by inverting the non-zero eigenvalues and leaving the zero eigenvalues unchanged in the eigendecomposition of $\mat{A}$. We use the big-O notation $\tilde{O}$ to ignore logarithm factors. Let $\mathcal{B}_2^d = \{ \vx \in \mathbb{R}^d : \lVert \vx \rVert_2 \leq 1 \}$ denote the $d$-dimensional ball.

%% file: Sections/results.tex
\section{Main Results \& Technical Overview}

\subsection{Algorithm Overview: Collaborative Personalized Phased elimination}

In this section, we briefly introduce our proposed algorithm in preparation for our main results in Section~\ref{subsec:main_result}. The regret analysis and algorithm details can be found in Section~\ref{sec:algorithm_analysis} and Section~\ref{app:alg}.

We propose Collaborative Personalized Phased Elimination (CP-PE) by adapting the method of phased elimination with $G$-optimal design exploration in~\cite{lattimore2020bandit} (Chapter 21) to our heterogeneous setting. A phased elimination algorithm shrinks the action set in each phase by eliminating the actions that are too far from the estimated parameter with a decreasing target error. Additionally, CP-PE consists of two stages: the \emph{collaborative learning stage} and the \emph{personalized learning stage}. A threshold is carefully designed to guide the transition between the two stages so that collaboration is activated only
when it benefits learning.

In the collaborative stage, all users perform the phased elimination algorithm together as if they have the same unknown parameter. This benefits learning when the target error is larger than the degree of heterogeneity since estimating $\vmu$ already gives a relatively good estimation to all $\vtheta_i$'s. However, as target error drops below the degree of heterogeneity, collaboratively learning an estimation on $\vmu$ does not give good enough estimations on $\vtheta_i$'s anymore. Thus, all users transit to the local stage, where they perform phased elimination on their own to refine their results.


\subsection{Main Results: Minimax-optimal Regret Bounds}
\label{subsec:main_result}
We provide our main result in the section. Let us consider the case where the covariance matrix is isotropic, i.e.,  $\mat{C} = \sigma^2 \mat{I}$ for some $\sigma \geq 0$. To present our results on the minimax lower bound, let us define the minimax regret.
\begin{definition}
    \label{def:minimax_reg}
    The minimax regret is defined as
    \begin{align*}
        \overline{R}_{m,n} = \inf_{\pi \in \Pi} \max_{\vmu \in \mathcal{B}_2^d} R_{m,n}(\vmu),
    \end{align*}
where $\Pi$ is the collection of all policies and $\mathcal{B}_2^d$ is the $d$-dimensional unit ball in Euclidean space.
\end{definition}
In the definition of the minimax regret, we maximize over all $\vmu$ in the unit ball, a discussion on the restriction can be found in remark~\ref{rmk:assump}.

The following Theorem~\ref{thm:main} combines our results on the minimax regret lower bound in Theorem~\ref{thm:lower_bound} and the regret analysis for CP-PE in Theorem~\ref{thm:alg_finite}. This gives us a complete characterization of the benefit of collaboration.

\begin{theorem}[Informal: A complete characterization on the benefit of collaboration]
    \label{thm:main}
    Consider the unit-ball action set $\mathcal{A} = \mathcal{B}^d_2$. Assume that $\vmu \in \mathcal{B}^d_2$ and $\sigma \leq 1$. With assumptions\footnote{The regret upper bounds holds for $n\ge md^2$. More details can be found in Theorem~\ref{thm:lower_bound} and Theorem~\ref{thm:alg_infinite}.} on the relation of $n$ and $d$, the minimax lower bound on the Bayesian regret is achieved by CP-PE omitting the logarithm terms. The regret bounds are shown in the following table:

    \begin{table}[h]
        \label{sample-table}
        \centering
        \begin{tabular}{llll}
            \toprule
            & $n = o \left( \frac{d}{m \sigma^2} \right)$     & $n = \frac{c d}{m^{2\gamma} \sigma^2}$ & $n = \omega \left( \frac{d}{\sigma^2} \right)$ \\
            \midrule
            The minimax regret  $\overline{R}_{m,n}$ & $\Omega(d\sqrt{mn})$ & $\Omega(dm^{1-\gamma}\sqrt{n})$ & $\Omega(dm\sqrt{n})$ \\
            Regret obtained by Alg.~\ref{alg:pb} & $\tilde{O}(d\sqrt{mn})$ & $\tilde{O}(dm^{1-\gamma}\sqrt{n})$ & $\tilde{O}(dm\sqrt{n})$ \\
            \bottomrule
        \end{tabular}
    \end{table}
    The columns are separated by three different regimes, where $c>0$ and $\gamma \in [0,1/2]$ are any constants. The first row is the minimax regret lower bound on the joint Bayes regret and the second row is the joint pseudo-regret obtained by CP-PE. See details in Theorem~\ref{thm:lower_bound} and Theorem~\ref{thm:alg_finite}.
\end{theorem}

\paragraph{Regret Characterization}
    \label{rmk:interpretation}
In Theorem~\ref{thm:main}, we identify three distinct regimes for the number of rounds $n$, characterized by their relationship to the heterogeneity level $\sigma$:
\textbf{data-scarce regime:} $n = o(dm^{-1}\sigma^{-2})$; \textbf{intermediate regime:} $n = \Theta(dm^{-2\gamma}\sigma^{-2})$; \textbf{data-abundant regime:} $n = \omega(d\sigma^{-2})$.
Equivalently, by fixing $n$, we can interpret the regimes in terms of heterogeneity:
\begin{itemize}
    \item \textbf{Data-scarce regime} corresponds to \textbf{low heterogeneity}, $\sigma = \omega(\sqrt{dm^{-1}n^{-1}})$
    \item \textbf{Intermediate regime} corresponds to $\sigma = \Theta(\sqrt{cdm^{-2\gamma}n^{-1}})$
    \item \textbf{Data-abundant regime} corresponds to \textbf{high heterogeneity}, $\sigma = \omega(\sqrt{dn^{-1}})$
\end{itemize}
The theorem shows that the regret behavior varies across these regimes, indicating whether collaboration among users is beneficial. Below, we elaborate on each regime.

\textbf{In the data-scarce regime}, where $n$ is small or $\sigma$ is low, we prove that no policy can achieve a regret smaller than $\Omega(d \sqrt{mn})$. This matches the minimax regret lower bound for the standard stochastic linear bandit with $mn$ total samples. This result can be interpreted in two ways. First, it shows that collaboration is highly beneficial when heterogeneity is sufficiently low: users can collectively estimate the shared population parameter $\vmu$, achieving performance comparable to the homogeneous case. Second, when $n$ is small, individual users lack sufficient data to accurately estimate their own parameters $\vtheta_i$, and hence benefit from jointly estimating $\vmu$ using all $mn$ samples.

\textbf{In the data-abundant regime}, where $n$ is large or $\sigma$ is high (i.e., the highly heterogeneous case), we show that any policy incurs regret at least $\Omega(dm \sqrt{n})$. This matches the minimax lower bound for the joint regret when each user solves their own linear bandit problem independently using $n$ samples, and the regrets are summed across all $m$ users. Again, this can be seen from two perspectives. First, it implies that collaboration offers little to no benefit in this regime due to the high heterogeneity. Alternatively, when $n$ is large, each user has enough data to accurately estimate their own parameter $\vtheta_i$, making collaboration unnecessary.

\textbf{The most interesting behavior arises in the intermediate regime}, where both the number of rounds and the heterogeneity level are moderate. We capture this regime using a smooth interpolation parameter $\gamma \in [0, 1/2]$, which governs the transition between the two extremes. Our results show that regret interpolates smoothly between the bounds in the data-scarce and data-abundant regimes, depending on $\gamma$. This reveals how the benefit of collaboration varies with both the level of heterogeneity $\sigma$ and the number of samples $n$. As $\gamma$ increases from $0$ to $1/2$, the problem gradually shifts from a collaboration-friendly regime to one where users are better off acting independently.

\begin{remark}[Global and individual optimality]
    \label{rmk:ind}

In Theorem~\ref{thm:main}, we study the joint cumulative regret. However, both our minimax lower bound and the regret upper bound achieved by CP-PEcan be directly translated into individual regret bounds by dividing the results by the number of users, $m$. In other words, we establish not only global optimality in terms of joint regret but also individual optimality for each user’s regret—a stronger result. More details can be found in Section~\ref{subsec:proof_finite}.


\end{remark}

\begin{remark}[Assumptions in Theorem~\ref{thm:main}]
    \label{rmk:assump}
    In Theorem~\ref{thm:main}, we restrict $\vmu$ to lie within the unit $\ell_2$ ball $\mathcal{B}_2^d$ and assume that $\sigma \leq 1$. These assumptions serve as analogs to the standard boundedness assumptions on the unknown parameter $\vtheta$ in the classic linear bandit setting. For example,~\cite[Chapter 19]{lai1985asymptotically} considers the condition $\sup_{\vx, \vx' \in \mathcal{A}} \langle \vtheta, \vx - \vx' \rangle \leq 1$, which, when the action set is the unit ball, is equivalent to assuming $\|\vtheta\|_2 \leq 2$.
In our multi-agent setting, where each $\vtheta_i$ is drawn from the population distribution $\mathcal{N}(\vmu, \sigma^2 \mat{I})$, we naturally place the boundedness assumptions on the population parameters $\vmu$ and $\sigma$ instead.
\end{remark}

%% file: Sections/lower_bound.tex
\section{Information-Theoretic Lower Bound}
\label{sec:lower_bound}

In this section, we provide details on our minimax lower bound result and a sketch proof highlighting the challenges and novel techniques.

\begin{theorem}[Minimax lower bound on the Bayesian regret]
    \label{thm:lower_bound}
    Consider the unit-ball action set $\mathcal{A} = \mathcal{B}^d_2$. Assume that $\vmu \in \mathcal{B}^d_2$ and $\sigma \leq 1$. Recall the minimax regret $\overline{R}_{m,n}$ in Definition~\ref{def:minimax_reg}. Then, we have the following results:
    \begin{itemize}
        \item For $n = o \left( \frac{d}{m \sigma^2} \right)$ and $mn \geq d$, we have $\overline{R}_{m,n} = \Omega \left( d \sqrt{mn} \right)$;
        \item For  $n = \Theta\left(\max\left[\frac{d}{m^{2\gamma} \sigma^2}, d\right]\right)$, we have $\overline{R}_{m,n} = \Omega \left( d m^{1-\gamma}  \sqrt{n} \right)$;

        
        \item For $n = \omega \left( \frac{d}{\sigma^2} \right)$, we have $\overline{R}_{m,n} = \Omega \left( d m \sqrt{n} \right)$. 
        
    \end{itemize}
\end{theorem}

The interpretation of the results on the regret is discussed in Remark~\ref{rmk:interpretation}, and we discuss the assumptions we made in the theorem in Remark~\ref{rmk:lower_assumptions}.

We provide a proof sketch of Theorem~\ref{thm:lower_bound} while keeping the detailed proof in Section~\ref{subsec:proof_lower}. In the data-scarce regime where $n = o(d/m\sigma^2)$, we will show that the heterogeneity is low such that the regret lower bound can not be far from the lower bound on the regret for the classic stochastic linear bandits with $mn$ samples. Thus, we get our result with the classic lower bound $\Omega(d\sqrt{mn})$ plus an additional term due to the heterogeneity, where the term is small enough such that the final bound stays in the order of $\Omega(d\sqrt{mn})$.

A technical challenge emerges in the data-rich regime. In the classic stochastic linear bandit problem, the Bayesian regret is defined as $R_n(\vtheta) = \mathbb{E} \left[ \sum_{t=1}^n \langle \vx^* - \vx_t, \vtheta \rangle \right]$ for any given $\vtheta$. A set of “bad” $\vtheta$'s is chosen to be the vertices of a $d$-dimensional hypercube. With the probabilistic method, it is shown that any policy must get a large regret on one of the $\vtheta$'s being the true parameter, and thus the minimax lower bound is shown to be $\Omega(d\sqrt{n})$ for any policy $\pi$ (\cite{lattimore2020bandit}, Section~24.2). The proof strongly depends on the fact that the chosen $\vtheta$ has the same absolute value in each coordinate and has a length in the order of $d/\sqrt{n}$.

However, in our heterogeneous setting, the parameters $\vtheta_i$'s are not arbitrary but follow the population distribution $\mathcal{N}(\vmu, \sigma^2 \mat{I})$. Thus, our minimax bound considers the worst-case $\vmu$ instead of a worst-case $\vtheta$. Especially in the data-rich regime when the heterogeneity is high, the norm of $\vtheta_i$'s will be larger than the order of $d/\sqrt{n}$ since $\sigma = \omega (\sqrt{d/n})$. Additionally, each $\vtheta_i$ can be in any direction following the population model. To address the challenges, we develop novel techniques to deal with $\vtheta_i$'s with large norms and pointing into arbitrary directions.

Due to symmetry, we can focus on analyzing the regret from one user, say, user $m$ in~\eqref{eq:lower_sym}. To address the issue of $\vtheta_m$ having a large norm, we choose the set of “bad” $\vtheta_m$'s to be the vertices of a $(d-1)$-dimensional hypercube in~\eqref{eq:lower_z} instead, which we denote them as $\{ \vtheta^{(\vz)}\}$ in the proof. Although $\vtheta^{(\vz)}$'s has a large norm, we sacrifice one of the $d$ coordinates to absorb most part of the norm and carefully construct them such that the distance between them is bad for the learner, meaning that the $\vtheta^{(\vz)}$'s are hard to distinguish so the regret is large, yet separated enough such that an agent cannot get a small regret by simply keep selecting the same action.

Now we are able to construct a set of “bad” $\vtheta^{(\vz)}$'s, we need them to marginally follow the population model as we mentioned. This is achieved by first sampling an arbitrary $\vtheta^{(0)}$ as in~\eqref{eq:lower_theta0} following the population model. Then, we form a set $\{\vtheta^{(\vz)}\}$ around $\vtheta^{(0)}$ as in~\eqref{eq:lower_angle} and Figure~\ref{fig:rotate}. 
Following the construction, each $\vtheta^{(\vz)}$ marginally follows the population model due to symmetry, yet they are artificially coupled so that they form the vertices of a $(d-1)$-dimensional hypercube. Since we consider the unit ball action set, we can rotate $\vtheta^{(\vz)}$ to align with the coordinates in~\eqref{eq:lower_rotate}.

Dealing with the large $\vtheta_i$'s and choosing the appropriate set of “bad” $\vtheta^{(\vz)}$'s are the most challenging parts in the proof. We achieve this by carefully choosing both the correct order of $n$ and the suitable coefficients to form the suitable angle between the $\vtheta^{(\vz)}$'s in~\eqref{eq:large_sigma_eta}. Unlike choosing the specific bad $\vtheta$'s in the regret lower bound for classic stochastic linear bandit problem~\cite{lattimore2020bandit} (Section 24.2), our proof also provides a way to show the lower bound when given the knowledge that the unknown parameter $\vtheta$'s are large. In other words, given the information that the norm of $\vtheta$ is large does not make the problem easier as larger $\vtheta$ might seem to be easier to distinguish. In contrast, our construction shows that the order of the regret remains the same.

\section{Proof of Theorem~\ref{thm:lower_bound}}
\label{subsec:proof_lower}

Let $\mathcal{A} = \mathcal{B}_d^2$ be the unit ball action set. We prove the lower bound in the three regimes separately. 
\subsubsection{Lower bound for large \texorpdfstring{$\sigma$}{TEXT}}
\label{subsec:low_large}
We start from the third regime, the data-rich regime, in Theorem~\ref{thm:lower_bound}. Equivalently, we have
\begin{align*}
    \sigma = \Theta \left( \sqrt{\frac{d}{n^{1-\alpha}}} \right)
\end{align*}
for some $\alpha \in [0,1]$.
By the definition of the regret, for any given $\vmu \in \mathbb{R}^d$, we have
\begin{align}
    R_{m,n}(\vmu)
    =&\; \mathbb{E} \left[ \sum_{i=1}^m \sum_{t=1}^n \langle \vx_i^* - \vx_{i,t}, \vtheta_i \rangle \right] \nonumber \\
    =&\; \mathbb{E} \left[ \sum_{i=1}^m \sum_{t=1}^n \lVert \vtheta_i \rVert_2 - \langle \vx_{i,t}, \vtheta_i \rangle \right] \nonumber \\
    =&\; \sum_{i=1}^m \mathbb{E} \left[ \sum_{t=1}^n \lVert \vtheta_i \rVert_2 - \langle \vx_{i,t}, \vtheta_i \rangle \right] \nonumber \\
    =&\; m \mathbb{E} \left[ \sum_{t=1}^n \lVert \vtheta_m \rVert_2 - \langle \vx_{m,t}, \vtheta_m \rangle \right]. \label{eq:lower_sym}
\end{align}
The second equality comes from the fact that $\max_{\vx \in \mathcal{B}_d^2} \langle \vx, \vtheta_i \rangle = \lVert \vtheta_i \rVert_2$ and the last equality is due to symmetry. Thus, let us focus on the term
\begin{align*}
    \mathbb{E} \left[ \sum_{t=1}^n \lVert \vtheta_m \rVert_2 - \langle \vx_{m,t}, \vtheta_m \rangle \right].
\end{align*}
Despite the expression, note that the actions $\{ \vx_{m,t} \}_{t=1:n}$ still depend on the parameters of all users, $\{ \vtheta_i \}_{i=1:m}$.
Consider the case when
\begin{align}
    \sigma = c_1 d^{\frac{1}{2}} n^{-\frac{1}{2} + \alpha} \label{eq:large_sigma_c_1}
\end{align}
for some constant $c_1 > 0$ and $\alpha \geq 0$. Notice that the results can be easily extended to the case where $c_1 d^{\frac{1}{2}} n^{-\frac{1}{2} + \alpha} \leq \sigma \leq c_1' d^{\frac{1}{2}} n^{-\frac{1}{2} + \alpha}$.

Suppose that we have a $d$-dimensional random vector $\vtheta^{(0)} \in \mathcal{N}(0,\sigma^2 \mat{I})$. We have the following lemma:
\begin{lemma}
    Let $z_1, \dots, z_d$ be $d$ independent random variables that follows the normal distribution
    \begin{align*}
        z_i \sim \mathcal{N}(0,\sigma^2) \quad \forall i \in [d]
    \end{align*}
    for some $\sigma > 0$. Then, we have
    \begin{align*}
        \mathbb{P} \left( \sum_{i=1}^d z_i^2 \in \left[ \frac{d \sigma^2}{2}, \frac{3 d \sigma^2}{2} \right] \right) \geq c_2
    \end{align*}
    for some constant $c_2 \in (0,1)$.
\end{lemma}
The lemma holds since $\sum_{i=1}^d z_i^2 / \sigma^2$ follows the chi-square distribution.

Let us consider the case that $\vmu = 0$. Since $\lVert \vtheta^{(0)} \rVert^2 = \sum_{\ell=1}^d ( \theta^{(0)}_\ell )^2$ and $\theta^{(0)}_\ell \sim \mathcal{N}(0,\sigma^2)$, by the lemma, we have
\begin{align*}
    \mathbb{P} \left( \lVert \vtheta^{(0)} \rVert^2 \in \left[ \frac{d \sigma^2}{2}, \frac{3d \sigma^2 }{2} \right] \right) \geq c_2
\end{align*}
with some constant $c_2 \in (0,1)$. That is,
\begin{align}
    \mathbb{P} \left( \lVert \vtheta^{(0)} \rVert \in \left[ \sigma \sqrt{\frac{d}{2}}, \sigma \sqrt{\frac{3d}{2}} \, \right] \right)
    =\mathbb{P} \left( \lVert \vtheta^{(0)} \rVert \in \left[ \frac{c_1 \sqrt{2}}{2} d n^{-\frac{1}{2} + \alpha},  \frac{c_1 \sqrt{6}}{2} d n^{-\frac{1}{2} + \alpha} \right] \right)
    \geq c_2. \label{eq:lower_theta0}
\end{align}
We denote the event by $\mathcal{E}$, i.e.,
\begin{align*}
    \mathcal{E} = \left\{ \lVert \vtheta^{(0)} \rVert \in \left[ \frac{c_1 \sqrt{2}}{2} d n^{-\frac{1}{2} + \alpha},  \frac{c_1 \sqrt{6}}{2} d n^{-\frac{1}{2} + \alpha} \right] \right\}
\end{align*}
and thus $\mathbb{P}(\mathcal{E}) \geq c_2$. Let $\mathcal{S}^{d-1}(\lVert \vtheta^{(0)} \rVert)$ be the sphere centered at $\pmb{0} \in \mathbb{R}^d$ with radius $\lVert \vtheta^{(0)} \rVert$ in the $d$ dimensional space. Let us define a distribution on a set of $d$-dimensional random vectors $\{ \vtheta^{(\vz)}: \vz \in \mathcal{Z} \}$ where the index set $\mathcal{Z} = \{ \pm 1 \}^{d-1}$. Given the event $\mathcal{E}$, let us write
\begin{align}
    \lVert \vtheta^{(0)} \rVert = c_3 d n^{- \frac{1}{2} + \alpha} \label{eq:lower_angle}
\end{align}
where $c_3 \in \left[ \frac{c_1 \sqrt{2}}{2}, \frac{c_1 \sqrt{6}}{2} \right]$. Let us denote
\begin{align*}
    \tilde{\mathcal{S}}^{d-2}_n(\vtheta^{(0)}) = \left\{ \vtheta \in \mathcal{S}^{d-1}(\lVert \vtheta^{(0}) \rVert): \frac{\vtheta^\top \vtheta^{(0)}}{\lVert \vtheta \rVert \lVert \vtheta^{(0)} \rVert} = \cos \eta \right\}
\end{align*}
where
\begin{gather}
    \eta = \sin^{-1} \left( \frac{c_4 \sqrt{d(d-1)} n^{-\frac{1}{2} + \frac{\alpha}{2}}}{\lVert \vtheta^{(0)} \rVert} \right)
    = \sin^{-1} \left( \frac{c_4}{c_3} \sqrt{\frac{d-1}{d n^\alpha}} \right) \in \left(0, \frac{\pi}{2} \right), \label{eq:large_sigma_eta} \\
    c_4 = \min \left\{ \frac{c_1 \sqrt{2}}{2}, \sqrt{\frac{c_1}{6\sqrt{2}}} \right\} \leq c_3. \nonumber
\end{gather}

\begin{figure}
    \centering
    \includegraphics[width=0.5\linewidth, trim={0 0 0 0},clip]{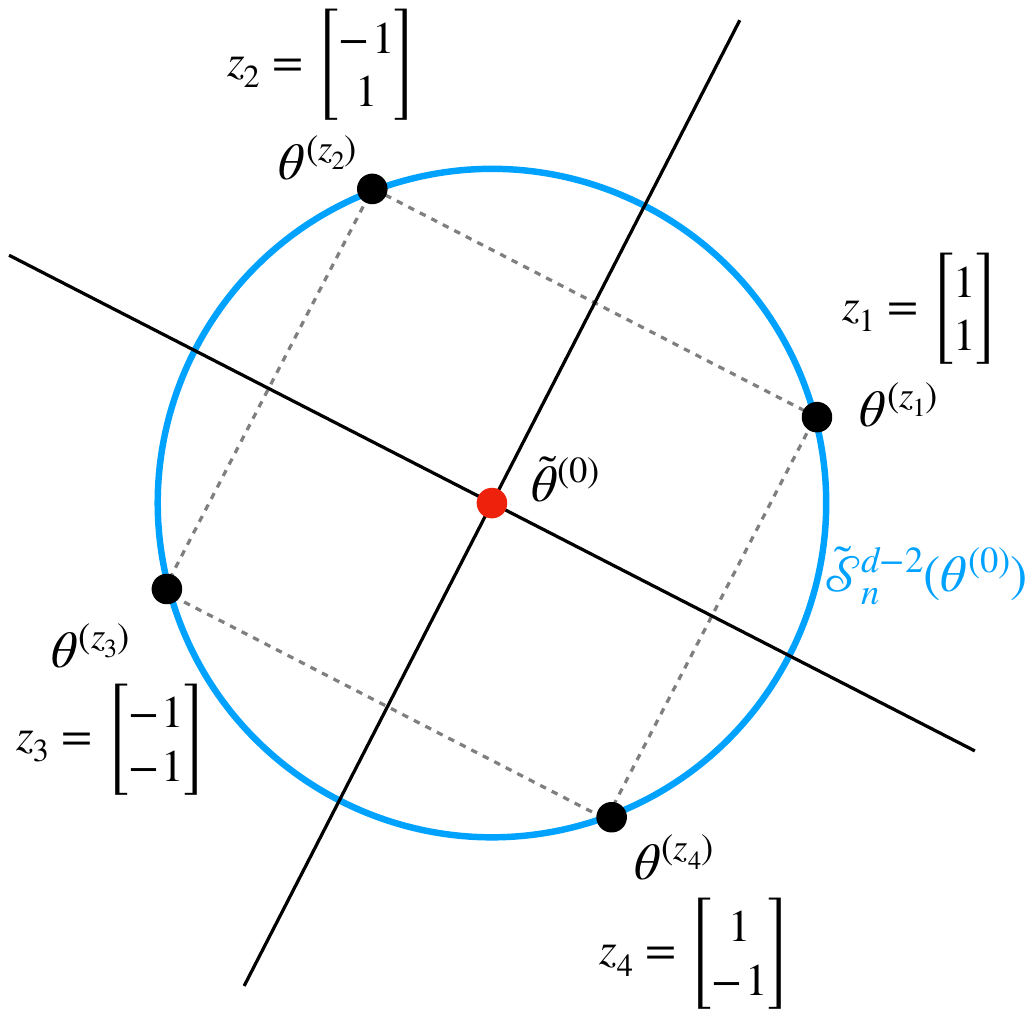}
    \caption{Constructing the $\vtheta^{(\vz)}$'s}
    \label{fig:rotate}
\end{figure}

The set $\tilde{\mathcal{S}}^{d-2}_n(\vtheta^{(0)})$ is a $(d-2)$-dimensional sphere on $\mathcal{S}^{d-1}(\lVert \vtheta^{(0)} \rVert)$. Let $\tilde{\vtheta}^{(0)}$ denote the center of $\tilde{\mathcal{S}}^{d-2}_n(\vtheta^{(0)})$. Notice that $\tilde{\vtheta}^{(0)} \neq \vtheta^{(0)}$. We want to randomly pick a set $\{ \vtheta^{(\vz)}: \vz \in \mathcal{Z} \} \subset \tilde{\mathcal{S}}^{d-2}_n(\vtheta^{(0)})$ such that $\{ \vtheta^{(\vz)}: \vz \in \mathcal{Z} \}$ is the set of vertices of a $(d-1)$-dimensional hyper-cube. To be more specific on the indices, each $(\vtheta^{(\vz)}, \vtheta^{(\vz')})$ pair forms the end points of an edge if $\lVert \vz - \vz' \rVert_1 = 2$, i.e., $\vz$ and $\vz'$ only differ in one coordinate.  To do so, we can simply start from uniformly choosing a $(d-1)$-dimensional basis, and find the $(d-1)$-dimensional hyper-cube on $\tilde{\mathcal{S}}^{d-2}_n(\vtheta^{(0)})$ that align with the basis. Due to the symmetry of sampling $\vtheta^{(0)}$ and the basis, we have
\begin{align}
    \vtheta^{(\vz)} \sim \mathcal{N}(0, \sigma^2 \mat{I}) \quad \forall \vz \in \mathcal{Z}. \label{eq:lower_z}
\end{align}
Now, with some rotation matrix $W \in \mathbb{R}^{d \times d}$, we can rotate $\vtheta^{(0)}$ and the hyper-cube $\{ \vtheta^{(\vz)}: z \in \mathcal{Z} \}$ such that $W\vtheta^{(0)}$ points toward the first coordinate and the hyper-cube aligns with the standard basis. Then, we have
\begin{align}
    \tilde{\vtheta}^{(\vz)} = W \vtheta^{(\vz)} \quad \forall z \in \mathcal{Z} \label{eq:lower_rotate}
\end{align}
where 
\begin{align*}
    \tilde{\theta}^{(\vz)}_\ell = 
    \begin{cases}
        \lVert \vtheta^{(0)} \rVert \cos \eta \quad& \ell = 1, \\
       \frac{z_\ell}{\sqrt{d-1}} \lVert \vtheta^{(0)} \rVert \sin \eta 
       = z_\ell c_4 \sqrt{d} n^{-\frac{1}{2} + \frac{\alpha}{2}} \quad& \ell = 2, \dots, d,
    \end{cases}
\end{align*}
and thus
\begin{align*}
    \lVert \tilde{\vtheta}^{(\vz_1)} - \tilde{\vtheta}^{(\vz_2)} \rVert = 2 c_4 n^{-\frac{1}{2} + \frac{\alpha}{2}} \quad \forall \vz_1, \vz_2 \in \mathcal{Z}, \lVert \vz_1 - \vz_2 \rVert_1 = 2.
\end{align*}

Thus, $\{\vtheta^{(\vz)}\}$ are the vertices of a $(d-1)$-dimensional hypercube aligning with the standard basis, and each of the vertices has the same value in the first coordinate.

Now we want to show a lower bound for the average regret with respect to $\{ \tilde{\vtheta}^{(\vz)}: \vz \in \mathcal{Z} \}$. For simplicity, let us denote $\Theta = \{ \vtheta^{(\vz)}: \vz \in \mathcal{Z} \}$. Suppose that $\vz$ is uniformly randomly sampled from $\{ \pm 1 \}^{d-1}$ and we set $\vtheta_m = \vtheta^{(\vz)}$. Since $\vtheta_m$ and $\vtheta^{(\vz)}$ follow the same distribution, we have
\begin{align}
    \mathbb{E} \left[ \sum_{t=1}^n \lVert \vtheta_m \rVert_2 - \langle \vx_{m,t}, \vtheta_m \rangle \right]
    =&\; \mathbb{E} \left[ \sum_{t=1}^n \lVert \vtheta^{(\vz)} \rVert_2 - \langle \vx_{m,t}, \vtheta^{(\vz)} \rangle \bigg| \vz \right] \nonumber \\
    =&\; \frac{1}{\lvert \mathcal{Z} \rvert}\sum_{\vz \in \mathcal{Z}} \mathbb{E} \left[ \sum_{t=1}^n \lVert \vtheta^{(\vz)} \rVert - \langle \vx_{m,t}, \vtheta^{(\vz)} \rangle \bigg| \vz \right] \nonumber \\
    =&\; \frac{1}{\lvert \mathcal{Z} \rvert}\sum_{\vz \in \mathcal{Z}} \mathbb{E} \left[ \sum_{t=1}^n \lVert \vtheta^{(\vz)} \rVert - \langle W^{-1} \vx_{m,t}, W \vtheta^{(\vz)} \rangle \bigg| \vz \right] \nonumber \\
    =&\; \frac{1}{\lvert \mathcal{Z} \rvert}\sum_{\vz \in \mathcal{Z}} \mathbb{E} \left[ \sum_{t=1}^n \lVert \vtheta^{(\vz)} \rVert - \langle \tilde{\vx}_{m,t}, \tilde{\vtheta}^{(\vz)} \rangle \bigg| \vz \right], \label{eq:large_sigma_reg_1}
\end{align}
where $\tilde{\vx}_{m,t}$ still falls in the unit ball due to the rotation.
For any $t \in [n]$ ,we have
\begin{align*}
    &\; \lVert \vtheta^{(\vz)} \rVert - \langle \tilde{\vx}_{m,t}, \tilde{\vtheta}^{(\vz)} \rangle \\
    =&\; \left( \lVert \vtheta^{(\vz)} \rVert \cos^2 \eta - \tilde{x}_{m,t,1} \tilde{\theta}^{(\vz)}_1 \right) + \sum_{\ell=2}^d \left( \lVert \vtheta^{(\vz)} \rVert \frac{\sin^2 \eta}{d-1} - \tilde{x}_{m,t,\ell} \tilde{\theta}^{(\vz)}_\ell \right) \\
    =&\; \left( \lVert \vtheta^{(\vz)} \rVert \cos^2 \eta - \tilde{x}_{m,t,1} \lVert \vtheta^{(\vz)} \rVert (\cos \eta) \right) \\
    &\; + \sum_{\ell=2}^d \left( \lVert \vtheta^{(\vz)} \rVert \cdot \frac{\sin^2 \eta}{d-1} - \tilde{x}_{m,t,\ell} \left( \lVert \vtheta^{(\vz)} \rVert \cdot \frac{\sin \eta}{\sqrt{d-1}} \right) z_\ell \right) \\
    =&\; \left( \frac{\lVert \vtheta^{(\vz)} \rVert \cos^2 \eta}{2} - \tilde{x}_{m,t,1} \lVert \vtheta^{(\vz)} \rVert \cos \eta + \frac{\lVert \vtheta^{(\vz)} \rVert \tilde{x}_{m,t,1}^2}{2} \right) + \left( \frac{\lVert \vtheta^{(\vz)} \rVert \cos^2 \eta}{2} - \frac{\lVert \vtheta^{(\vz)} \rVert \tilde{x}_{m,t,1}^2}{2} \right) \\
    &\; + \sum_{\ell=2}^d \left( \frac{\lVert \vtheta^{(\vz)} \rVert \sin^2 \eta}{2(d-1)} - \tilde{x}_{m,t,\ell} \left( \lVert \vtheta^{(\vz)} \rVert \cdot \frac{\sin \eta}{\sqrt{d-1}} \right) z_\ell + \frac{\lVert \vtheta^{(\vz)} \rVert \tilde{x}_{m,t,\ell}^2}{2} \right) \\
    &\; + \sum_{\ell=2}^d \left( \frac{\lVert \vtheta^{(\vz)} \rVert \sin^2 \eta}{2(d-1)} - \frac{\lVert \vtheta^{(\vz)} \rVert \tilde{x}_{m,t,\ell}^2}{2} \right) \\
    =&\; \frac{\lVert \vtheta^{(\vz)} \rVert \cos^2 \eta}{2} \left( 1 - 2 \frac{\tilde{x}_{m,t,1}}{\cos \eta} + \frac{\tilde{x}_{m,t,1}^2}{\cos^2 \eta} \right) + \left( \frac{\lVert \vtheta^{(\vz)} \rVert \cos^2 \eta}{2} - \frac{\lVert \vtheta^{(\vz)} \rVert \tilde{x}_{m,t,1}^2}{2} \right) \\
    &\; + \sum_{\ell=2}^d \frac{\lVert \vtheta^{(\vz)} \rVert \sin^2 \eta}{2} \left( \frac{1}{d-1} - \frac{2 \tilde{x}_{m,t,\ell} z_\ell}{\sqrt{d-1} \sin \eta} + \frac{\tilde{x}_{m,t,\ell}^2}{\sin^2 \eta} \right) \\
    &\; + \sum_{\ell=2}^d \left( \frac{\lVert \vtheta^{(\vz)} \rVert \sin^2 \eta}{2(d-1)} - \frac{\lVert \vtheta^{(\vz)} \rVert \tilde{x}_{m,t,\ell}^2}{2} \right) \\
    =&\; \frac{\lVert \vtheta^{(\vz)} \rVert \cos^2 \eta}{2} \left( 1 - \frac{\tilde{x}_{m,t,1}}{\cos \eta} \right)^2 + \sum_{\ell=2}^d \frac{\lVert \vtheta^{(\vz)} \rVert \sin^2 \eta}{2} \left( \frac{1}{\sqrt{d-1}} - \frac{\tilde{x}_{m,t,\ell}}{\sin \eta} \right)^2 \\
    &\; + \left( \frac{\lVert \vtheta^{(\vz)} \rVert}{2} - \frac{\lVert \vtheta^{(\vz)} \rVert}{2} \sum_{\ell=1}^d \tilde{x}_{m,t,\ell}^2 \right) \\
    \geq&\; \sum_{\ell=2}^d \frac{\lVert \vtheta^{(\vz)} \rVert \sin^2 \eta}{2} \left( \frac{1}{\sqrt{d-1}} - \frac{\tilde{x}_{m,t,\ell} \cdot z_\ell}{\sin \eta} \right)^2.
\end{align*}
Thus, we have
\begin{align}
    &\; \mathbb{E} \left[ \sum_{t=1}^n \lVert \vtheta^{(\vz)} \rVert - \langle \tilde{\vx}_{i,t}, \tilde{\vtheta}^{(\vz)} \rangle \bigg| \vz \right] \nonumber \\
    \geq&\; \mathbb{P}( \mathcal{E}) \, \mathbb{E} \left[ \sum_{t=1}^n \sum_{\ell=2}^d \frac{\lVert \vtheta^{(\vz)} \rVert \sin^2 \eta}{2} \left( \frac{1}{\sqrt{d-1}} - \frac{\tilde{x}_{m,t,\ell} \cdot z_\ell}{\sin \eta} \right)^2 \bigg| \vz, \mathcal{E} \right] \nonumber \\
    \geq&\; c_2 \mathbb{E} \left[ \sum_{\ell=2}^d \frac{\lVert \vtheta^{(\vz)} \rVert \sin^2 \eta}{2} \sum_{t=1}^{n'} \left( \frac{1}{\sqrt{d-1}} - \frac{\tilde{x}_{m,t,\ell} \cdot z_\ell}{\sin \eta} \right)^2 \bigg| \vz, \mathcal{E} \right] \nonumber \\
    =&\; c_2 \mathbb{E} \left[ \sum_{\ell=2}^d \frac{c_4^2 d(d-1) n^{-1+\alpha}}{2 \lVert \vtheta^{(0)} \rVert} \sum_{t=1}^{n'} \left( \frac{1}{\sqrt{d-1}} - \frac{\tilde{x}_{m,t,\ell} \cdot z_\ell}{\sin \eta} \right)^2 \bigg| \vz, \mathcal{E} \right] \nonumber \\
    =&\; c_2 \mathbb{E} \left[ \sum_{\ell=2}^d \frac{c_4^2 d(d-1) n^{-1+\alpha}}{2 c_3 d n^{-\frac{1}{2}+\alpha}} \sum_{t=1}^{n'} \left( \frac{1}{\sqrt{d-1}} - \frac{\tilde{x}_{m,t,\ell} \cdot z_\ell}{\sin \eta} \right)^2 \bigg| \vz, \mathcal{E} \right] \nonumber \\
    =&\; \frac{c_2 (d-1) }{2 \sqrt{n}} \sum_{\ell=2}^d \mathbb{E} \left[ \frac{c_4^2}{c_3} \sum_{t=1}^{n'} \left( \frac{1}{\sqrt{d-1}} - \frac{\tilde{x}_{m,t,\ell} \cdot z_\ell}{\sin \eta} \right)^2 \bigg| \vz, \mathcal{E} \right] \nonumber \\
    =&\; \frac{c_2 (d-1) }{2 \sqrt{n}} \sum_{\ell=2}^d \mathbb{E} \left[ \frac{c_4^2}{c_3 \sin^2 \eta} \sum_{t=1}^{n'} \left( \frac{\sin \eta}{\sqrt{d-1}} - \tilde{x}_{m,t,\ell} \cdot z_\ell \right)^2 \bigg| \vz, \mathcal{E} \right] \nonumber \\
    =&\; \frac{c_2 (d-1) }{2 \sqrt{n}} \sum_{\ell=2}^d \mathbb{E} \left[ \frac{c_3 d n^\alpha}{d-1} \sum_{t=1}^{n'} \left( \frac{\sin \eta}{\sqrt{d-1}} - \tilde{x}_{m,t,\ell} \cdot z_\ell \right)^2 \bigg| \vz, \mathcal{E} \right] \nonumber \\
    \geq&\; \frac{c_2 d n^{-\frac{1}{2} + \alpha}}{2} \sum_{\ell=2}^d \mathbb{E} \left[ \frac{\sqrt{2} c_1}{2} \sum_{t=1}^{n'} \left( \frac{\sin \eta}{\sqrt{d-1}} - \tilde{x}_{m,t,\ell} \cdot z_\ell \right)^2 \bigg| \vz, \mathcal{E} \right] \nonumber \\
    =&\; c_5 d n^{-\frac{1}{2} + \alpha} \sum_{\ell=2}^d \vz, \mathbb{E} \left[ \sum_{t=1}^{n'} \left( \frac{\sin \eta}{\sqrt{d-1}} - \tilde{x}_{m,t,\ell} \cdot z_\ell \right)^2 \bigg| \mathcal{E} \right] \label{eq:large_sigma_reg_2}
\end{align}
where $c_5 = (\sqrt{2} c_1 c_2)/4$ and
\begin{gather*}
    n' = \min \left\{ n, \min \left\{\tau : \sum_{t=1}^\tau \tilde{x}_{m,t,\ell}^2 \geq \frac{n \sin^2 \eta}{d-1 } \right\} \right\}.
\end{gather*}
Now, we want to show that give a $(\vz, \vz')$ pair such that $\vz$ and $\vz'$ only differ in the $\ell$-th coordinate, the regret
\begin{align*}
    \mathbb{E} \left[ \sum_{t=1}^{n'} \left( \frac{\sin \eta}{\sqrt{d-1}} - \tilde{x}_{m,t,\ell} \cdot z_\ell \right)^2 \bigg| \vz, \mathcal{E} \right] + \mathbb{E} \left[ \sum_{t=1}^{n'} \left( \frac{\sin \eta}{\sqrt{d-1}} - \tilde{x}_{m,t,\ell} \cdot z'_\ell \right)^2 \bigg| \vz', \mathcal{E} \right]
\end{align*}
is bounded by some lower bound.
For any $w \in \{1,-1\}$, we have
\begin{align*}
    \sum_{t=1}^{n'} \left( \frac{\sin \eta}{\sqrt{d-1}} - \tilde{x}_{m,t,\ell} \cdot w \right)^2
    \leq&\; 2 \sum_{t=1}^{n'} \frac{\sin^2 \eta}{d-1} + 2 \sum_{t=1}^{n'} \tilde{x}_{m,t,\ell}^2 \\
    \leq&\; \frac{2n\sin^2 \eta}{d-1} + 2 \left( \frac{n\sin^2 \eta}{d-1} + 1 \right) \\
    =&\; \frac{4n}{d-1} \cdot \frac{c_4^2}{c_3^2} \cdot \frac{d-1}{d n^\alpha} + 2 \\
    \leq&\; \frac{4 c_4^2 n^{1-\alpha}}{c_3^2 d} + 2 \\
    \leq&\; \frac{4 (c_4^2/c_3^2) n^{1-\alpha} + 2d}{d} \\
    \leq&\; \frac{6 (c_4^2/c_3^2) n^{1-\alpha}}{d}.
\end{align*}
Following the proof in Section 24.2 of~\cite{lattimore2020bandit}, with Pinsker's inequality and Exercise 14.4 in~\cite{lattimore2020bandit}, we have
\begin{align*}
    &\; \left\lvert \mathbb{E} \left[ \sum_{t=1}^{n'} \left( \frac{\sin \eta}{\sqrt{d-1}} - \tilde{x}_{m,t,\ell} \right)^2 \Bigg| \vz', \mathcal{E} \right]  - \mathbb{E} \left[ \sum_{t=1}^{n'} \left( \frac{\sin \eta}{\sqrt{d-1}} - \tilde{x}_{m,t,\ell} \right)^2 \Bigg| \vz', \mathcal{E} \right] \right\rvert \\
    \leq&\; \left( \frac{8 (c_4^2/c_3^2) n^{1-\alpha}}{d} \right) \sqrt{\frac{1}{2} \mathrm{D}\left( \mathbb{P}_{\vz} \Vert \mathbb{P}_{\vz'} \right)},
\end{align*}
where $\mathbb{P}_{\vz}$ is the joint distribution of $\{(\vx_{i,t}, y_{i,t})\}_{i=1:m, t=1:n}$ given that $\vtheta_m = \vtheta^{(\vz)}$ and event $\mathcal{E}$ holds.
With Section 24.2 and Exercise 15.7 in~\cite{lattimore2020bandit}, we know that
\begin{align*}
    \mathrm{D}\left( \mathbb{P}_{\vz} \Vert \mathbb{P}_{\vz'} \right)
    =&\; \frac{1}{2} \mathbb{E}_{\vz} \left[ \sum_{t=1}^{n'} \left\langle \tilde{\vx}_{m,t}, \tilde{\vtheta}^{(\vz)} - \tilde{\vtheta}^{(\vz')} \right\rangle^2 \right] \\
    =&\; \frac{1}{2} \mathbb{E}_{\vz} \left[ \sum_{t=1}^{n'} \left( x_{m,t,\ell} \cdot \left( \theta^{(\vz)}_\ell - \theta^{(\vz')}_\ell \right) \right)^2 \right] \\
    \leq&\; \frac{1}{2} \left(2  c_4 \sqrt{d} n^{-\frac{1}{2} + \frac{\alpha}{2}} \right)^2 \mathbb{E}_{\vz} \left[ \sum_{t=1}^{n'} x_{m,t,\ell}^2 \right] \\
    \leq&\; \frac{1}{2} \left( 2 c_4 \sqrt{d} n^{-\frac{1}{2} + \frac{\alpha}{2}} \right)^2 \left( \frac{n \sin^2 \eta}{d-1} + 1 \right) \\
    =&\; 2 c_4^2 d n^{- 1 + \alpha} \left( \frac{n}{d-1} \frac{c_4^2}{c_3^2} \frac{d-1}{d n^\alpha} + 1 \right) \\
    \leq&\; \frac{2 c_4^4}{c_3^2} + 2 c_4^2 d n^{-1+\alpha} \\
    \leq&\; \frac{2 c_4^4}{c_3^2} + \frac{2 c_4^4}{c_3^2} \\
    \leq&\; \frac{4 c_4^4}{c_3^2}.
\end{align*}
With the bound on the KL-divergence term, we have
\begin{align*}
    &\; \left\lvert \mathbb{E}_{\vz} \left[ \sum_{t=1}^{n'} \left( \frac{\sin \eta}{\sqrt{d-1}} - \tilde{x}_{m,t,\ell} \right)^2 \right]  - \mathbb{E}_{\vz'} \left[ \sum_{t=1}^{n'} \left( \frac{\sin \eta}{\sqrt{d-1}} - \tilde{x}_{m,t,\ell} \right)^2 \right] \right\rvert \\
    \leq&\; \left( \frac{6 (c_4^2/c_3^2) n^{1-\alpha}}{d} \right) \sqrt{\frac{4 c_4^4}{c_3^2}} \\
    =&\; \frac{12 (c_4^4/c_3^3) n^{1-\alpha}}{d}.
\end{align*}
Thus, we have
\begin{align*}
    \mathbb{E}_{\vz} \left[ \sum_{t=1}^{n'} \left( \frac{\sin \eta}{\sqrt{d-1}} - \tilde{x}_{m,t,\ell} \right)^2 \right]
    \geq&\; \mathbb{E}_{\vz'} \left[ \sum_{t=1}^{n'} \left( \frac{\sin \eta}{\sqrt{d-1}} - \tilde{x}_{m,t,\ell} \right)^2 \right] - \frac{12 (c_4^4/c_3^3) n^{1-\alpha}}{d}.
\end{align*}
Adding an extra term on the both side gives us
\begin{align*}
    &\; \mathbb{E}_{\vz} \left[ \sum_{t=1}^{n'} \left( \frac{\sin \eta}{\sqrt{d-1}} - \tilde{x}_{m,t,\ell} \right)^2 \right] + \mathbb{E}_{\vz'} \left[ \sum_{t=1}^{n'} \left( \frac{\sin \eta}{\sqrt{d-1}} + \tilde{x}_{m,t,\ell} \right)^2 \right] \\
    \geq&\; \mathbb{E}_{\vz'} \left[ \sum_{t=1}^{n'} \left( \frac{\sin \eta}{\sqrt{d-1}} - \tilde{x}_{m,t,\ell} \right)^2 \right] + \mathbb{E}_{\vz'} \left[ \sum_{t=1}^{n'} \left( \frac{\sin \eta}{\sqrt{d-1}} + \tilde{x}_{m,t,\ell} \right)^2 \right] - \frac{12 n^{1-\alpha}}{d} \\
    =&\; 2 \sum_{t=1}^{n'} \frac{\sin^2 \eta}{d-1} + 2 \mathbb{E}_{\vz'} \left[ \sum_{t-1}^{n'} \tilde{x}_{m,t,\ell} \right] - \frac{12 n^{1-\alpha}}{d} \\
    \geq&\; \frac{2 n \sin^2 \eta}{d-1} - \frac{12 n^{1-\alpha}}{d} \\
    =&\; \frac{2n}{d-1} \frac{c_4^2}{c_3^2} \frac{d-1}{d n^\alpha} - \frac{12 n^{1-\alpha}}{d} \\
    =&\; \frac{2 c_4^2 n^{1-\alpha}}{c_3^2 d} \left( 1 - \frac{6 c_4^2}{c_3} \right).
\end{align*}
Since we have
\begin{align*}
    c_4^2
    = \frac{c_1}{6 \sqrt{2}}
    \leq \frac{(\sqrt{2} c_3/2)}{6 \sqrt{2}}
    \leq \frac{c_3}{12},
\end{align*}
we have
\begin{align*}
    &\; \mathbb{E}_{\vz} \left[ \sum_{t=1}^{n'} \left( \frac{\sin \eta}{\sqrt{d-1}} - \tilde{x}_{m,t,\ell} \right)^2 \right] + \mathbb{E}_{\vz'} \left[ \sum_{t=1}^{n'} \left( \frac{\sin \eta}{\sqrt{d-1}} + \tilde{x}_{m,t,\ell} \right)^2 \right] \\
    \geq&\; \frac{c_4^2 n^{1-\alpha}}{c_3^2 d} \\
    \geq&\; \min \left\{ \frac{c_1^2}{2}, \frac{c_1}{6\sqrt{2}} \right\} \cdot \left( \frac{\sqrt{6}}{2} c_1 \right)^{-2} \frac{n^{1-\alpha}}{d} \\
    =&\; \frac{c_6 n^{1-\alpha}}{d}
\end{align*}
where $c_6 = \min \left\{ \frac{1}{3}, \frac{1}{9\sqrt{2} c_1} \right\}$. Similarly, we also have
\begin{align*}
    &\; \mathbb{E}_{\vz} \left[ \sum_{t=1}^{n'} \left( \frac{\sin \eta}{\sqrt{d-1}} + \tilde{x}_{m,t,\ell} \right)^2 \right] + \mathbb{E}_{\vz'} \left[ \sum_{t=1}^{n'} \left( \frac{\sin \eta}{\sqrt{d-1}} - \tilde{x}_{m,t,\ell} \right)^2 \right]
    \geq \frac{c_6 n^{1-\alpha}}{d}.
\end{align*}
These two give us
\begin{align*}
    &\; \mathbb{E}_{\vz} \left[ \sum_{t=1}^{n'} \left( \frac{\sin \eta}{\sqrt{d-1}} - \tilde{x}_{m,t,\ell} \cdot z_\ell \right)^2 \right] + \mathbb{E}_{\vz'} \left[ \sum_{t=1}^{n'} \left( \frac{\sin \eta}{\sqrt{d-1}} - \tilde{x}_{m,t,\ell} \cdot z'_\ell \right)^2 \right]
    \geq \frac{c_6 n^{1-\alpha}}{d}.
\end{align*}
Plug this to~\eqref{eq:large_sigma_reg_1} and~\eqref{eq:large_sigma_reg_2}, we have
\begin{align*}
    &\; \mathbb{E} \left[ \sum_{t=1}^n \lVert \vtheta_m \rVert_2 - \langle \vx_{i,t}, \vtheta_m \rangle \right] \\
    =&\; \frac{1}{\lvert \mathcal{Z} \rvert} \sum_{\vz \in \mathcal{Z}} \mathbb{E} \left[ \sum_{t=1}^n \lVert \vtheta^{(\vz)} \rVert - \langle \tilde{\vx}_{i,t}, \tilde{\vtheta}^{(\vz)} \rangle \right] \\
    \geq&\; \frac{1}{\lvert \mathcal{Z} \rvert} \sum_{\vz \in \mathcal{Z}} c_5 d n^{-\frac{1}{2} + \alpha} \sum_{\ell=2}^d \mathbb{E} \left[ \sum_{t=1}^{n'} \left( \frac{\sin \eta}{\sqrt{d-1}} - \tilde{x}_{m,t,\ell} \cdot z_\ell \right)^2 \bigg| \mathcal{E} \right] \\
    =&\; c_5 d n^{-\frac{1}{2} + \alpha} \sum_{\ell=2}^d \frac{1}{\lvert \mathcal{Z} \rvert} \sum_{\vz \in \mathcal{Z}} \mathbb{E} \left[ \sum_{t=1}^{n'} \left( \frac{\sin \eta}{\sqrt{d-1}} - \tilde{x}_{m,t,\ell} \cdot z_\ell \right)^2 \bigg| \mathcal{E} \right] \\
    =&\; c_5 d n^{-\frac{1}{2} + \alpha} \sum_{\ell=2}^d \frac{1}{2 \lvert \mathcal{Z} \rvert} \sum_{\vz \in \mathcal{Z}} \mathbb{E} \left[ \sum_{t=1}^{n'} \left( \frac{\sin \eta}{\sqrt{d-1}} - \tilde{x}_{m,t,\ell} \cdot z_\ell \right)^2 \bigg| \mathcal{E} \right] + \mathbb{E} \left[ \sum_{t=1}^{n'} \left( \frac{\sin \eta}{\sqrt{d-1}} - \tilde{x}_{m,t,\ell} \cdot z'_\ell \right)^2 \bigg| \mathcal{E} \right]  \\
    \geq&\; c_5 d n^{-\frac{1}{2} + \alpha} \sum_{\ell=2}^d \frac{1}{2 \lvert \mathcal{Z} \rvert} \sum_{\vz \in \mathcal{Z}} \frac{c_6 n^{1-\alpha}}{d} \\
    =&\; \frac{c_5 c_6}{2} (d-1) \sqrt{n} \\
    \geq&\; c_7 d \sqrt{n},
\end{align*}
where $c_7 = (c_5 c_6)/ 4$.

The assumption we use is that
\begin{align*}
    d \leq n^{1-\alpha} \cdot \min \left\{ \frac{1}{3}, \frac{1}{9 \sqrt{2} c_1} \right\}
\end{align*}

\subsubsection{Middle regime}
For the second regime in Theorem~\ref{thm:lower_bound}, the proof is very similar to the proof for large $\sigma$. We just need an extra $m^{-\gamma}$ term and most parts of the proof stay the same. In the middle regime, instead of~\eqref{eq:large_sigma_c_1}, we now have
\begin{align*}
    \sigma = c_1 d^{\frac{1}{2}} m^{-\gamma}n^{-\frac{1}{2}}
\end{align*}
for $\gamma \in \left[0, \frac{1}{2} \right]$ and some constant $c_1 > 0$. Then, we have
\begin{gather*}
    \eta = \sin^{-1} \left( \frac{c_4 \sqrt{d(d-1)} m^{-\gamma} n^{-\frac{1}{2}}}{\lVert \vtheta^{(0)} \rVert} \right)
    = \sin^{-1} \left( \frac{c_4}{c_3} \sqrt{\frac{d-1}{d}} \right) \in \left(0, \frac{\pi}{2} \right).
\end{gather*}
compared to~\eqref{eq:large_sigma_eta}. 
Then, in~\eqref{eq:large_sigma_reg_2}, we have the extra $m^{-\gamma}$ term in the front,
\begin{align*}
    &\; \mathbb{E} \left[ \sum_{t=1}^n \lVert \vtheta^{(\vz)} \rVert - \langle \tilde{\vx}_{i,t}, \tilde{\vtheta}^{(\vz)} \rangle \bigg| \vz \right] \\
    \geq&\; \mathbb{P}( \mathcal{E}) \, \mathbb{E} \left[ \sum_{t=1}^n \sum_{\ell=2}^d \frac{\lVert \vtheta^{(\vz)} \rVert \sin^2 \eta}{2} \left( \frac{1}{\sqrt{d-1}} - \frac{\tilde{x}_{m,t,\ell} \cdot z_\ell}{\sin \eta} \right)^2 \bigg| \vz, \mathcal{E} \right] \\
    \geq&\; c_2 \mathbb{E} \left[ \sum_{\ell=2}^d \frac{\lVert \vtheta^{(\vz)} \rVert \sin^2 \eta}{2} \sum_{t=1}^{n'} \left( \frac{1}{\sqrt{d-1}} - \frac{\tilde{x}_{m,t,\ell} \cdot z_\ell}{\sin \eta} \right)^2 \bigg| \vz, \mathcal{E} \right] \\
    =&\; c_2 \mathbb{E} \left[ \sum_{\ell=2}^d \frac{c_4^2 d(d-1) m^{-2\gamma} n^{-1+\alpha}}{2 \lVert \vtheta^{(0)} \rVert} \sum_{t=1}^{n'} \left( \frac{1}{\sqrt{d-1}} - \frac{\tilde{x}_{m,t,\ell} \cdot z_\ell}{\sin \eta} \right)^2 \bigg| \vz, \mathcal{E} \right] \\
    =&\; c_2 \mathbb{E} \left[ \sum_{\ell=2}^d \frac{c_4^2 d(d-1) m^{-2\gamma} n^{-1+\alpha}}{2 c_3 d m^{-\gamma} n^{-\frac{1}{2}+\alpha}} \sum_{t=1}^{n'} \left( \frac{1}{\sqrt{d-1}} - \frac{\tilde{x}_{m,t,\ell} \cdot z_\ell}{\sin \eta} \right)^2 \bigg| \vz, \mathcal{E} \right] \\
    =&\; \frac{c_2 (d-1) }{2 m^\gamma \sqrt{n}} \sum_{\ell=2}^d \mathbb{E} \left[ \frac{c_4^2}{c_3} \sum_{t=1}^{n'} \left( \frac{1}{\sqrt{d-1}} - \frac{\tilde{x}_{m,t,\ell} \cdot z_\ell}{\sin \eta} \right)^2 \bigg| \vz, \mathcal{E} \right] \\
    =&\; \frac{c_2 (d-1) }{2 m^\gamma \sqrt{n}} \sum_{\ell=2}^d \mathbb{E} \left[ \frac{c_4^2}{c_3 \sin^2 \eta} \sum_{t=1}^{n'} \left( \frac{\sin \eta}{\sqrt{d-1}} - \tilde{x}_{m,t,\ell} \cdot z_\ell \right)^2 \bigg| \vz, \mathcal{E} \right] \\
    =&\; \frac{c_2 (d-1) }{2 m^\gamma \sqrt{n}} \sum_{\ell=2}^d \mathbb{E} \left[ \frac{c_3 d}{d-1} \sum_{t=1}^{n'} \left( \frac{\sin \eta}{\sqrt{d-1}} - \tilde{x}_{m,t,\ell} \cdot z_\ell \right)^2 \bigg| \vz, \mathcal{E} \right] \\
    \geq&\; \frac{c_2 d n^{-\frac{1}{2}}}{2 m^\gamma} \sum_{\ell=2}^d \mathbb{E} \left[ \frac{\sqrt{2} c_1}{2} \sum_{t=1}^{n'} \left( \frac{\sin \eta}{\sqrt{d-1}} - \tilde{x}_{m,t,\ell} \cdot z_\ell \right)^2 \bigg| \vz, \mathcal{E} \right] \\
    =&\; c_5 d m^{-\gamma} n^{-\frac{1}{2}} \sum_{\ell=2}^d \mathbb{E} \left[ \sum_{t=1}^{n'} \left( \frac{\sin \eta}{\sqrt{d-1}} - \tilde{x}_{m,t,\ell} \cdot z_\ell \right)^2 \bigg| \vz, \mathcal{E} \right].
\end{align*}
The rest of the proof are the same. Due to the $m^{-\gamma}$ term in the front, we have
\begin{align*}
    \mathbb{E} \left[ \sum_{i=1}^m \sum_{t=1}^n \lVert \vtheta_m \rVert_2 - \langle \vx_{i,t}, \vtheta_m \rangle \right]
    = \Omega \left( d m^{1-\gamma} \sqrt{n} \right).
\end{align*}

\subsection{Regret lower bound for the first regime}
Equivalently we have
\begin{align*}
    \sigma = o \left( \sqrt{\frac{d}{mn}} \right).
\end{align*}
Let $\Pi$ be the set of all policies
\begin{align*}
    \Pi = \left\{ \pi : \pi = (\pi_{i,t})_{i=1:m, t=1:n}, \pi_{i,t}: \left\{ (\vx_{i,\tau}, y_{i,\tau}) \right\}_{i=1:m, \tau=1:t-1} \mapsto \vx_t \right\}
\end{align*}
and
\begin{align*}
    \vx_i^* = \arg\max_{\vx \in \mathcal{A}} \langle \vx, \vtheta_i \rangle, \\
    \vx^* = \arg\max_{\vx \in \mathcal{A}} \langle \vx, \vmu \rangle.
\end{align*}
Then,
\begin{align}
    R_{m,n}(\vmu)
    =&\; \min_{\pi \in \Pi} \mathbb{E} \left[ \sum_{i=1}^m \sum_{t=1}^n \langle \vx_i^* - \vx_{i,t}, \vtheta_i \rangle \right] \nonumber \\
    \geq&\; \min_{\pi \in \Pi} \mathbb{E} \left[ \sum_{i=1}^m \sum_{t=1}^n \langle \vx^* - \vx_{i,t}, \vtheta_i \rangle \right] \nonumber \\
    =&\; \min_{\pi \in \Pi} \mathbb{E} \left[ \sum_{i=1}^m \sum_{t=1}^n \langle \vx^* - \vx_{i,t}, \vmu \rangle + \sum_{i=1}^m \sum_{t=1}^n \langle \vx^* - \vx_{i,t}, \vtheta_i - \vmu \rangle \right] \nonumber \\
    \geq&\; \min_{\pi \in \Pi} \mathbb{E} \left[ \sum_{i=1}^m \sum_{t=1}^n \langle \vx^* - \vx_{i,t}, \vmu \rangle \right] + \min_{\pi \in \Pi} \mathbb{E} \left[ \sum_{i=1}^m \sum_{t=1}^n \langle \vx^* - \vx_{i,t}, \vtheta_i - \vmu \rangle \right] \nonumber \\
    \geq&\; \min_{\pi \in \Pi} \mathbb{E} \left[ \sum_{i=1}^m \sum_{t=1}^n \langle \vx^* - \vx_{i,t}, \vmu \rangle \right] - \left\lvert \min_{\pi \in \Pi} \mathbb{E} \left[ \sum_{i=1}^m \sum_{t=1}^n \langle \vx^* - \vx_{i,t}, \vtheta_i - \vmu \rangle \right] \right\rvert. \label{eq:small_sigma_regret}
\end{align}
Since $\sigma = o \left( \sqrt{d/mn} \right)$, we have
\begin{align}
    \left\lvert \mathbb{E} \left[ \sum_{i=1}^m \sum_{t=1}^n \langle \vx^* - \vx_{i,t}, \vtheta_i - \vmu \rangle \right] \right\rvert
    \leq&\; \mathbb{E} \left[ \sum_{i=1}^m \sum_{t=1}^n \lVert \vx^* - \vx_{i,t} \rVert_2 \lVert \vtheta_i - \vmu \rVert_2 \right] \nonumber \\
    \leq&\; 2 m n \mathbb{E} \left[ \lVert \vtheta_i - \vmu \rVert_2 \right] \nonumber \\
    \leq&\; 2 m n \sqrt{ \mathbb{E} \left[ \lVert \vtheta_i - \vmu \rVert_2^2 \right] } \nonumber \\
    =&\; 2mn \sqrt{d\sigma^2} \nonumber \\
    =&\; o \left(d\sqrt{mn} \right). \label{eq:small_sigma_shift}
\end{align}
Now, let $\Pi'$ be the set of policies that knows $\vu_1, \dots, \vu_m$, that is,
\begin{align*}
    \Pi = \left\{ \pi : \pi = (\pi_{i,t})_{i=1:m, t=1:n}, \pi_{i,t}:( \{ \vu_i \}_{i=1:m}, \left\{ (\vx_{i,\tau}, y_{i,\tau}) \right\}_{i=1:m, \tau=1:t-1} ) \mapsto \vx_t \right\}.
\end{align*}
Since $\Pi \subset \Pi'$, we have
\begin{align*}
    R_{m,n}(\vmu)
    \geq&\; \min_{\pi \in \Pi'} \mathbb{E} \left[ \sum_{i=1}^m \sum_{t=1}^n \langle \vx^* - \vx_{i,t}, \vmu \rangle \right] - 2mn \sqrt{d\sigma^2}.
\end{align*}
Since $\pi \in \Pi'$ knows $\vu_1, \dots, \vu_m$, equivalently it knows
\begin{align*}
    y'_{i,t}
    = \langle \vx_{i,t}, \vmu \rangle + \epsilon_{i,t}
    = \langle \vx_{i,t}, \vtheta_i \rangle + \langle \vx_{i,t}, \vmu - \vtheta_i \rangle + \epsilon_{i,t}
    = y_{i,t} - \langle \vx_{i,t}, \vu_i \rangle.
\end{align*}
Thus, we must have
\begin{align*}
    \min_{\pi \in \Pi'} \mathbb{E} \left[ \sum_{i=1}^m \sum_{t=1}^n \langle \vx^* - \vx_{i,t}, \vmu \rangle \right]
\end{align*}
obeying the lower bound for the conventional bandits problem where there are $mn$ samples and $\vmu$ is the true parameter. From Section 24.2 in~\cite{lattimore2020bandit}, we have that
\begin{align}
    \min_{\pi \in \Pi'} \mathbb{E} \left[ \sum_{i=1}^m \sum_{t=1}^n \langle \vx^* - \vx_{i,t}, \vmu \rangle \right]
    \geq \frac{1}{32} d \sqrt{mn} \label{eq:small_sigma_classic}
\end{align}
for some carefully chosen $\vmu \in \mathcal{B}_2^d$. Plug~\eqref{eq:small_sigma_classic} and~\eqref{eq:small_sigma_shift} back to~\eqref{eq:small_sigma_regret} and we have
\begin{align*}
    R_{m,n}(\vmu) = \Omega \left( d \sqrt{mn} \right).
\end{align*}

\begin{remark}[Assumptions made in Theorem~\ref{thm:lower_bound}]
    \label{rmk:lower_assumptions}
    In Theorem~\ref{thm:lower_bound}, we made minor assumptions on the relation between the total number of rounds, $n$, and the dimension, $d$. This naturally follows the assumptions on $\vmu$ and $\sigma$, which are discussed in remark~\ref{rmk:assump}. In the first regime, we assume that $d \leq mn$. This corresponds to the classic minimax bound result~\cite{lattimore2020bandit} (Theorem~24.4), where we assume that the dimension is smaller than the number of samples. In the second regimes, we the assumptions follow the identical order as in the classic result that $d \leq n$. In the third regime, we assume that $n = \Omega(d^{-(1-\gamma)})$. This naturally follows the assumption we made on $\vmu$ and $\sigma$ being at most constant. In this case, we must have $d = O(n^{1-\gamma})$ in order to get $\sigma = O(1)$. The constants is the assumption are adjusted to deal with the technical proofs.
\end{remark}

%% file: Sections/alg_regret.tex
\section{Algorithm: Collaborative Personalized Phased Elimination}
\label{sec:algorithm_analysis}

In this section, we provide a regret analysis for Algorithm~\ref{alg:pb}. The algorithm is composed of two stages: the collaborative learning stage and the personalized learning stage. We present Algorithm~\ref{alg:pb} and the two composing blocks: the Collaborative-Phased Elimination block~\ref{alg:pb_collab} and the Phased Elimination block~\ref{alg:pb_pers}.

\begin{algorithm}[b]
    \caption{Collaborative Personalized Phased Elimination (CP-PE)}
    \label{alg:pb}
    \begin{algorithmic}[1]
        \STATE \textbf{Input}: action set $\mathcal{A}$, dimension $d$, number of users $m$, number of rounds $n$, level of heterogeneity $\sigma$, $\delta \in (0,1)$
        \STATE $\mathcal{A}_{i,1} \leftarrow \mathcal{A} \quad \forall i \in [m], \, \ell \leftarrow 0$, calculate $h_{\sigma,\delta}$ from~\eqref{eq:h} 
        \WHILE{total rounds for each user $\leq n$}
            \STATE $\varepsilon \leftarrow 2^{-\ell}$
            \IF{$h_{\sigma, \delta} \leq \frac{1}{2} \varepsilon_\ell$} 
                \STATE $\{ \mathcal{A}_{i,\ell+1} \}_{i=1:m} \leftarrow \text{Collaborative-Phased Elimination}(\mathcal{A}_{1,\ell}, \varepsilon_\ell)$
            \ELSE
                \FOR{$i \in [m]$}
                    \STATE $\mathcal{A}_{i,\ell+1} \leftarrow \text{Phased Elimination}(\mathcal{A}_{i,\ell}, \varepsilon_\ell)$
                \ENDFOR
            \ENDIF
        \ENDWHILE
    \end{algorithmic}
\end{algorithm}

\begin{algorithm}
    \caption{Collaborative-Phased Elimination($\mathcal{A}_{1,\ell}$, $\varepsilon_\ell$)}
    \label{alg:pb_collab}
    \begin{algorithmic}[1]
        \STATE Find $G$-optimal design $\vpi_\ell \in \mathcal{P}(\mathcal{A}_{1,\ell})$ with $\mathrm{Supp}(\vpi_\ell) \leq d(d+1) / 2$ 
        \FOR{$i \in [m]$}
            \STATE User $i$ selects each action $\vx \in \mathcal{A}_{i,\ell}$ for $n_{i,\ell}(\vx)$ times as defined in~\eqref{eq:alg_n_col}.
            \STATE Record the actions $\mat{X}_i \in \mathbb{R}^{\bar{n}_{i,\ell} \times d}$ and rewards $\vy_i \in \mathbb{R}^{\bar{n}_{i,\ell}}$.
        \ENDFOR
        \STATE Calculate $\hat{\vmu}_\ell = \frac{1}{m} \sum_{i=1}^m ({\mat{X}_i}^\top {\mat{X}_i})^{\dagger} {\mat{X}_i}^\top \vy_i$ and update
        \begin{align*}
            \mathcal{A}_{i,\ell+1} \leftarrow \left\{ x \in \mathcal{A}_{\ell} : \max_{\vx' \in \mathcal{A}_{1,\ell}} \langle \hat{\vmu}_\ell, \vx' - \vx \rangle \leq 2 \varepsilon_\ell \right\} \quad \forall i \in [m].
        \end{align*}
        \RETURN $\{ \mathcal{A}_{i,\ell+1} \}_{i=1:m}$
    \end{algorithmic}
\end{algorithm}

\begin{algorithm}
    \caption{Phased Elimination($\mathcal{A}_{i,\ell}$, $\varepsilon_\ell$)}
    \label{alg:pb_pers}
    \begin{algorithmic}[1]
        \STATE Find $G$-optimal design $\vpi_{i,\ell} \in \mathcal{P}(\mathcal{A}_{i,\ell})$ with $\mathrm{Supp}(\vpi_{i,\ell}) \leq d(d+1) / 2$
        \STATE Select each action $\vx \in \mathcal{A}_{i,\ell}$ exactly $n_{i,\ell}(x)$ times as defined in~\eqref{eq:alg_n_loc}.
        \STATE Record the actions $\mat{X}_i \in \mathbb{R}^{\bar{n}_{i,\ell} \times d}$ and rewards $\vy_i \in \mathbb{R}^{\bar{n}_{i,\ell}}$.
        \STATE Calculate $\hat{\vtheta}_{i,\ell} \leftarrow (\mat{X}_i^\top \mat{X}_i)^{\dagger} \mat{X}_i^\top \vy_i$ and update
        \begin{align*}
            \mathcal{A}_{i,\ell+1} \leftarrow \left\{ \vx \in \mathcal{A}_{i,\ell} : \max_{\vx' \in \mathcal{A}_{i,\ell}} \langle \hat{\vtheta}_{i,\ell}, \vx' - \vx \rangle \leq 2 \varepsilon_\ell \right\}
        \end{align*}
        \RETURN $\mathcal{A}_{i,\ell+1}$
    \end{algorithmic}
\end{algorithm}


In the collaborative learning stage (line~5 and~6 in Algorithm~\ref{alg:pb}), all users start with the same action set $\mathcal{A}$ and perform phased elimination together. In each phase $\ell$, we solve the G-optimal design for the current action set $\mathcal{A}_{1,\ell}$ and let $\vpi_{\ell}$ be the design solution, which is a probability distribution on the actions in $\mathcal{A}_{1,\ell}$. Let $n_{i,\ell}(\vx)$ denote the number of times for the action $\vx$ to be selected by user $i$ in the phase and choose
\begin{align}
    n_{i,\ell}(\vx) = \left\lceil \frac{\vpi_\ell(\vx) g(\vpi_\ell)}{m \varepsilon_\ell^2} \log \frac{k\ell(\ell+1)}{2\delta} \right\rceil \label{eq:alg_n_col}
\end{align}
for all $\vx \in \mathcal{A}_{i,\ell}$ and $i \in [m]$. Let $\bar{n}_{i,\ell} = \sum_{\vx} n_{i.\ell}(\vx)$ denote the total number of pulls done by user $i$ in phase $\ell$. Notice that the factor $m$ in the denominator is the benefit of collaboration that each user only needs to afford $1/m$ of the pulls it would need if it were performing phased elimination on its own. In Lemma~\ref{lem:upp_1}, we see that by selecting this amount of actions, the agents collaboratively learn a good estimation on $\vmu$ with the target error $\varepsilon_\ell$. And as we decrease $\varepsilon_\ell$ in each phase, we gradually form a better estimation on $\vmu$.

However, the ultimate goal for each user $i$ is to find the best action with respect to $\vtheta_i$ instead of $\vmu$. This motivates our second stage, the personalized learning stage. In this stage (line~7 to~9 in Algorithm~\ref{alg:pb}), each user continues from the action set it gets from the end of the collaboration stage and performs phased elimination by its own. In each phase $\ell$, we now have
\begin{align}
    n_{i,\ell}(\vx) = \left\lceil \frac{\vpi_\ell(\vx) g(\vpi_\ell)}{\varepsilon_\ell^2} \log \frac{km\ell(\ell+1)}{2\delta} \right\rceil \label{eq:alg_n_loc}
\end{align}
for all $\vx \in \mathcal{A}_{i,\ell}$. Also in Lemma~\ref{lem:upp_1}, we can see that this gives each user $i$ an estimation on $\vtheta_i$ with an error $\varepsilon_\ell$. Thus, each user can now refine its own result and find its optimal action. 

The timing to switch from the collaborative stage to the personalized learning stage depends on the amount of heterogeneity. In the data-scarce regime, staying in the collaborative stage increases the effective number of samples for each user and benefits the learning. Also, when the heterogeneity is low, it reduces the workload for each user while providing a good estimation for each $\vtheta_i$, since $\vtheta_i$'s and $\vmu$ in this regime. In the data-rich regime or with high heterogeneity, the users need to switch to the personalized learning stage earlier since collaboration barely helps. The timing to switch stages is explicitly determined by the term
\begin{align}
    h_{\sigma, \delta} = \sqrt{ 2 \left( \max_{\vx \in \mathcal{A}} \lVert \vx \rVert_{\mat{C}}^2 \right) \log \frac{4k}{\delta} } + \sqrt{ \frac{2}{m} \left( \max_{\vx \in \mathcal{A}} \lVert \vx \rVert_{\mat{C}}^2 \right) \log \frac{4k}{\delta} } \label{eq:h}
\end{align}
for any given $\delta \in (0,1)$. We can see that the term $h_{\sigma, \delta}$ depends on the covariance matrix $\mat{C}$, which is the covariance matrix in the population distribution reflecting the heterogeneity of the users. We can see from line~3 in Algorithm~\ref{alg:pb} that the users stay in the collaborative stage if the target error $\varepsilon_\ell$ is larger than $2h_{\sigma,\delta}$, and switch to the personalized learning stage once the target error goes below $2h_{\sigma,\delta}$. This switching is natural here since estimating $\vmu$ is good enough for estimating $\vtheta_i$ as long as the target error $\varepsilon_\ell$ surpassed the heterogeneity. On the other hand, when the target error is becoming sufficiently small, estimating $\vmu$ will never give us a good estimation on $\vtheta_i$, and thus we need to switch to personalized learning stage instead.

For the regret analysis on Algorithm~\ref{alg:pb}, we start from the case with finite action sets in Theorem~\ref{thm:alg_finite}, and extend the results for the unit-ball action set presented in Theorem~\ref{thm:alg_infinite}.

We extend the result on finite action sets to the unit-ball action set.

\begin{theorem}
    \label{thm:alg_infinite}
    Consider the unit ball action set $\mathcal{A} = \left\{ \vx: \lVert \vx \rVert_2 \leq 1 \right\}$, and assume that $\lVert \vmu \rVert_2 \leq 1$, $\sigma \leq 1$, and $n \geq d^2$. For any $\delta \in (0,1)$, by running Algorithm~\ref{alg:pb}, we have the following results:
    \begin{itemize}
        \item When $n = o \left( d / m \sigma^2 \right)$, we have $\hat{R}_{m,n}(\vmu) = \tilde{O} \left( d\sqrt{mn} + md^2 \right)$.
        
        \item When $n = \Theta \left( d / m^{2\gamma} \sigma^2 \right)$ with $\gamma \in [0, \frac{1}{2}]$, we have $\hat{R}_{m,n}(\vmu) = \tilde{O} \left( d m^{1-\gamma} \sqrt{n} + md^2 \right)$.
        
        \item When $n = \omega \left( d / \sigma^2 \right)$, we have $\hat{R}_{m,n}(\vmu) = \tilde{O} \left( dm \sqrt{n} + md^2 \right)$.
    \end{itemize}
    with probability at least $1-\delta$.
\end{theorem}

The interpretation of the regret bounds can be found in Remark~\ref{rmk:interpretation}.

\begin{remark}[Assumptions in Theorem~\ref{thm:alg_finite} and~\ref{thm:alg_infinite}]
    We discuss the assumption on $\vmu$ and $\sigma$ in Remark~\ref{rmk:assump}. The assumptions on the relations of $n$ and $d$ come from the phased elimination algorithm. In each phase $\ell$, each action in the support of the design $\vpi_\ell$ has to be selected at least once. Since the support of the design is bounded by the size of the action set, or by $d(d+1)/2$ if the action set is compact (\cite{lattimore2020bandit}, Theorem~21.1), we have the assumptions in the theorems.
\end{remark}

\begin{remark}[The $md^2$ term in the upper bound]
    We discuss the additional $md^2$ term in detail in Section~\ref{app:upper_md2}. Notice that the term is a constant with respect to the number of rounds $n$. As long as we have $n \geq md^2$, the second term is always dominated by the first term in the regret bound and thus we omit the term in Theorem~\ref{thm:main}.
\end{remark}

\begin{proof}[Sketch proofs of Theorem~\ref{thm:alg_finite} and Theorem~\ref{thm:alg_infinite}]
    Following Algorithm~\ref{alg:pb}, by choosing each action for the number of times indicated in~\eqref{eq:alg_n_col} and~\eqref{eq:alg_n_loc}, we show that each agent $i$ learns a good estimation on $\vtheta_i$ with the target error $\varepsilon_\ell$ in Lemma~\ref{lem:upp_1}. With the lemma, the following Lemma~\ref{lem:upp_2} and~\ref{lem:upp_3} are relatively standard from the proof in the single-agent case. With the lemmas, we rewrite the regret upper bound into an optimization form in~\eqref{eq:upper_opt}. Then, we optimize~\eqref{eq:upper_opt} separately for the three regimes and get our final bounds. To extend the finite action set result to the unit ball action set case, we adopt the standard $\epsilon$-net argument. The detailed proofs can be found in Section~\ref{subsec:proof_finite} and~\ref{subsec:proof_infinite}.
\end{proof}

\section{Regret Analysis for CP-PE}
\label{app:alg}

\subsection{Regret analysis of Algorithm~\ref{alg:pb} for finite action sets}
\label{app:thm:finite}
\begin{theorem}
    \label{thm:alg_finite}
    Consider any finite action set $\mathcal{A}$ with $\lvert \mathcal{A} \rvert = k$, $k \in \mathbb{N}$, and assume that $\lVert \vmu \rVert_2 \leq 1$, $\sigma \leq 1$, and $n \geq \min \{k, d^2\}$. For any $\delta \in (0,1)$, by running Algorithm~\ref{alg:pb}, we have the following results:
    \begin{itemize}
        \item When $n = o \left( d / m \sigma^2 \right)$, we have $\hat{R}_{m,n}(\vmu) = \tilde{O} \left( \sqrt{dmn} + md^2 \right)$.
        
        \item When $n = \Theta \left( d / m^{2\gamma} \sigma^2 \right)$ with $\gamma \in [0, \frac{1}{2}]$, we have $\hat{R}_{m,n}(\vmu) = \tilde{O} \left( m^{1-\gamma} \sqrt{dn} + md^2 \right)$.
        
        \item When $n = \omega \left( d / \sigma^2 \right)$, we have $\hat{R}_{m,n}(\vmu) = \tilde{O} \left( m \sqrt{dn} + md^2 \right)$.
    \end{itemize}
    with probability at least $1-\delta$.
\end{theorem}

\subsection{Proof of Theorem~\ref{thm:alg_finite}}
\label{subsec:proof_finite}

For notational convenience, denote by $\hat{\vtheta}_{1,\ell} = \dots = \hat{\vtheta}_{1,\ell} = \hat{\vmu}_\ell$ during the collaborative learning stage, i.e.,  when $h_{\sigma, \delta} \leq \frac{1}{2} \varepsilon_\ell$.
\begin{lemma}
    \label{lem:upp_1}
    Following Algorithm~\ref{alg:pb}, we have
    \begin{align*}
        \mathbb{P} \left( \forall \ell \in \mathbb{N}, \; \forall i \in [m], \; \forall \vx \in \mathcal{A}_{i,\ell} : \lvert \langle \hat{\vtheta}_{i,\ell} - \vtheta_i, \vx \rangle \rvert \leq \varepsilon_\ell \right) \geq 1 - \delta.
    \end{align*}
\end{lemma}
\begin{proof}[Proof of Lemma~\ref{lem:upp_1}]
    Suppose that the algorithm has $L_c \in \mathbb{N}$ collaboration phases. That is,
    \begin{align*}
        L_c = \max \left\{ 0, \left\lfloor \log_2 \frac{1}{2 h_{\sigma, \delta}} \right\rfloor \right\}.
    \end{align*}
    In the collaboration phases, for any $\ell \in [L_c]$ and action $\vx \in \mathcal{A_{\ell}}$, we have
    \begin{align}
        \langle \hat{\vmu}_\ell - \vtheta_i, \vx \rangle
        = \langle \hat{\vmu}_\ell - \vmu, \vx \rangle + \langle \vmu - \vtheta_i, \vx \rangle.\label{eq:collab_rad}
    \end{align}
    Let $\mat{X}_i$ be a $\bar{n}_{i,\ell}$ by $d$ matrix such that each row of $\mat{X}_i$ is the actions chosen by user~$i$ at phase $\ell$, and $\vy_i$ is the vector with corresponding rewards. For simplicity, we omit the index $\ell$ in the notation. For the first term in~\eqref{eq:collab_rad}, we have
    \begin{align}
        \langle \hat{\vmu}_\ell - \vmu, \vx \rangle
        =&\; \left\langle \frac{1}{m} \sum_{i=1}^m \left( \mat{X}_i^\top \mat{X}_i \right)^{\dagger} \mat{X}_i^\top \vy_i -  \vmu, \vx \right\rangle \nonumber \\
        =&\; \left\langle \frac{1}{m} \sum_{i=1}^m \left( \mat{X}_i^\top \mat{X}_i \right)^{\dagger} \mat{X}_i^\top \left( \mat{X}_i (\vmu + \vu_i) + \vepsilon_i \right) -  \vmu, \vx \right\rangle \nonumber \\
        =&\; \left\langle \frac{1}{m} \sum_{i=1}^m \left( \mat{X}_i^\top \mat{X}_i \right)^{\dagger} \mat{X}_i^\top \mat{X}_i \vmu - \vmu, \vx \right\rangle + \left\langle \frac{1}{m} \sum_{i=1}^m \left( \mat{X}_i^\top \mat{X}_i \right)^{\dagger} \mat{X}_i^\top \mat{X}_i \vu_i, \vx \right\rangle \\
        &\; + \left\langle \frac{1}{m} \sum_{i=1}^m \left( \mat{X}_i^\top \mat{X}_i \right)^{\dagger} \mat{X}_i^\top \vepsilon_i, \vx \right\rangle \label{eq:collab_rad1}
    \end{align}
    For the first term in~\eqref{eq:collab_rad1}, we have
    \begin{align*}
        \left\langle \left( \mat{X}_i^\top \mat{X}_i \right)^{\dagger} \mat{X}_i^\top \mat{X}_i \vmu, \vx \right\rangle
        = \left\langle \vmu, \mat{X}_i^\top \mat{X}_i \left( \mat{X}_i^\top \mat{X}_i \right)^{\dagger} \vx \right\rangle
        = \left\langle \vmu, \vx \right\rangle
    \end{align*}
    and thus
    \begin{align*}
        \left\langle \frac{1}{m} \sum_{i=1}^m \left( \mat{X}_i^\top \mat{X}_i \right)^{\dagger} \mat{X}_i^\top \mat{X}_i \vmu - \vmu, \vx \right\rangle = 0.
    \end{align*}
    For the second term in~\eqref{eq:collab_rad1}, we have
    \begin{align*}
        \left\langle \frac{1}{m} \sum_{i=1}^m \left( \mat{X}_i^\top \mat{X}_i \right)^{\dagger} \mat{X}_i^\top \mat{X}_i \vu_i, \vx \right\rangle
        = \frac{1}{m} \sum_{i=1}^m \left\langle \vu_i, \mat{X}_i^\top \mat{X}_i \left( \mat{X}_i^\top \mat{X}_i \right)^{\dagger} \vx \right\rangle
        = \frac{1}{m} \sum_{i=1}^m \left\langle \vu_i, \vx \right\rangle.
    \end{align*}
    Since each $\vu_i \sim \mathcal{N}(\pmb{0}, \mat{C})$, we have
    \begin{align}
        \left\langle \frac{1}{m} \sum_{i=1}^m \vu_i, \vx \right\rangle \sim \mathcal{N} \left( 0,\frac{1}{m} \vx^\top \mat{C} \vx \right), \label{eq:upper_proof_gaus}
    \end{align}
    and thus
    \begin{align}
        \mathbb{P} \left( \left\langle \frac{1}{m} \sum_{i=1}^m \left( \mat{X}_i^\top \mat{X}_i \right)^{\dagger} \mat{X}_i^\top \mat{X}_i \vu_i, \vx \right\rangle \geq s \right)
        \leq \frac{1}{2} \exp \left\{ - \frac{m s^2}{2 \vx^\top \mat{C} \vx} \right\} \label{eq:upper_proof_gaus_concen}
    \end{align}
    for any $s \geq 0$. Thus, we have
    \begin{align}
        \mathbb{P} \left( \exists \ell \in [L_c] , \; \exists \vx \in \mathcal{A}_{\ell}: \left| \left\langle \frac{1}{m} \sum_{i=1}^m \left( \mat{X}_i^\top \mat{X}_i \right)^{\dagger} \mat{X}_i^\top \mat{X}_i \vu_i, \vx \right\rangle \right| \geq \sqrt{ \frac{2}{m} \left( \max_{\vx \in \mathcal{A}} \lVert \vx \rVert_\mat{C}^2 \right) \log \frac{4k}{\delta} } \right)
        \leq \frac{\delta}{4} \label{eq:collab_rad11}
    \end{align}
    for any $\delta > 0$. For the third term in~\eqref{eq:collab_rad1}, since $\mat{X}_i$ and $\vepsilon_i$ are independent, we have
    \begin{align*}
        \frac{1}{m} \sum_{i=1}^m \left( \mat{X}_i^\top \mat{X}_i \right)^{\dagger} \mat{X}_i^\top \vepsilon_i
        \sim \mathcal{N} \left( 0, \frac{1}{m^2} \sum_{i=1}^m \left( \mat{X}_i^\top \mat{X}_i \right)^{\dagger} \right).
    \end{align*}
    Then, we have
    \begin{align*}
        \mathbb{P} \left( \exists \vx \in \mathcal{A}_{\ell} : \left| \left\langle\frac{1}{m} \sum_{i=1}^m \left( \mat{X}_i^\top \mat{X}_i \right)^{\dagger} \mat{X}_i^\top \vepsilon_i, \vx \right\rangle \right| \geq \sqrt{ \frac{2 \sigma^2}{m^2} \lVert \vx \rVert_{\sum_{i=1}^m \left( \mat{X}_i^\top \mat{X}_i \right)^{\dagger}}^2 \log \frac{2 k \ell (\ell+1)}{\delta} } \right) \leq \frac{\delta}{2 \ell (\ell+1)}
    \end{align*}
    for any $\delta > 0$. 
    From the algorithm, we choose the actions $\mat{X}_i$ such that
    \begin{align*}
        &\; \frac{2}{m^2} \lVert \vx \rVert_{ \sum_{i=1}^m \left( \mat{X}_i^\top \mat{X}_i \right)^{\dagger}}^2 \log \frac{2 k \ell (\ell+1)}{\delta} \\
        \leq&\; \frac{2}{m^2} \vx^\top \sum_{i=1}^m \left( \left\lceil \frac{8 g(\pi_\ell)}{m \varepsilon_\ell^2} \log \frac{2 k \ell (\ell+1)}{\delta} \right\rceil \sum_{\vx' \in \mathcal{A}_{i,\ell}} \pi_\ell(\vx') \vx'^\top \vx' \right)^{\dagger} \vx \log \frac{2 k \ell (\ell+1)}{\delta} \\
        \leq&\; \frac{\varepsilon_\ell^2}{4}.
    \end{align*}
    Thus, we have
    \begin{align}
        \forall \ell \in [L_c] : \quad \mathbb{P} \left( \exists \vx \in \mathcal{A}_{\ell} : \left| \left\langle\frac{1}{m} \sum_{i=1}^m \left( \mat{X}_i^\top \mat{X}_i \right)^{\dagger} \mat{X}_i^\top \vepsilon_i, \vx \right\rangle \right| \geq \frac{\vepsilon_\ell}{2} \right) \leq \frac{\delta}{2 \ell (\ell+1)}. \label{eq:collab_rad12}
    \end{align}
    For the second term in~\eqref{eq:collab_rad}, we have
    \begin{align*}
        \langle \vmu - \vtheta_i, \vx \rangle
        = \vx^\top \vu_i
        \sim \mathcal{N}( 0, \vx^\top \mat{C} \vx ).
    \end{align*}
    Thus, for any $s \geq 0$,
    \begin{align}
        \mathbb{P} \left( \langle \vmu - \vtheta_i, \vx \rangle \geq s \right)
        \leq \frac{1}{2} \exp \left\{ - \frac{s^2}{2 \vx^\top \mat{C} \vx} \right\}. \label{eq:upper_proof_gaus_concen2}
    \end{align}
    Therefore, we have
    \begin{align} 
        \mathbb{P} \left( \exists \ell \in [L_c], \; \exists i \in [m], \; \exists \vx \in \mathcal{A}_\ell : \left| \langle \vmu - \vtheta_i, \vx \rangle \right| \geq \sqrt{ 2 \left( \max_{\vx \in \mathcal{A}} \lVert \vx \rVert_\mat{C}^2 \right) \log \frac{4km}{\delta} } \right)
        \leq \frac{\delta}{4} \label{eq:collab_rad2}
    \end{align}
    for any $\delta > 0$. In the collaborations, we know that
    \begin{align*}
        h_{\sigma, \delta} = \sqrt{ \frac{2}{m} \left( \max_{\vx \in \mathcal{A}} \lVert \vx \rVert_\mat{C}^2 \right) \log \frac{4k}{\delta} } + \sqrt{ 2 \left( \max_{\vx \in \mathcal{A}} \lVert \vx \rVert_\mat{C}^2 \right) \log \frac{4k}{\delta} }
        \leq \frac{\vepsilon_\ell}{2}.
    \end{align*}
    Plug~\eqref{eq:collab_rad11},~\eqref{eq:collab_rad12} and~\eqref{eq:collab_rad2} back to~\eqref{eq:collab_rad1} and we have
    \begin{align}
        \mathbb{P} \left( \exists \ell \in [L_c], \; \exists i \in [m], \; \exists \vx \in \mathcal{A}_{\ell} : \langle \hat{\vmu}_\ell - \vtheta_i, \vx \rangle \geq \vepsilon_\ell \right)
        \leq \left(1 - \frac{1}{L_c+1} \right) \delta. \label{eq:collab_rad_fin1}
    \end{align}
    Now, for the local training phases, we know that for any $\vx \in \mathcal{A}_{i,\ell}$
    \begin{align*}
        \hat{\vtheta}_{i,\ell} - \vtheta_i
        = \left( \mat{X}_i^\top \mat{X}_i \right)^{\dagger} \mat{X}_i^\top \vepsilon_i
        \sim \mathcal{N} \left( 0, \left( \mat{X}_i^\top \mat{X}_i \right)^{\dagger} \right)
    \end{align*}
    Therefore, we have
    \begin{align*}
        \mathbb{P} \left( \exists i \in [m], \; \exists \vx \in \mathcal{A}_{i,\ell} : \left| \langle \hat{\vtheta}_{i,\ell} - \vtheta_i, \vx \rangle \right| \geq \sqrt{2 \lVert \vx \rVert_{\left( \mat{X}_i^\top \mat{X}_i \right)^{\dagger}}^2 \log \frac{2km\ell (\ell+1)}{\delta} } \right) \leq \frac{\delta}{2 \ell (\ell+1)}
    \end{align*}
    for any $\delta>0$. From the algorithm, we choose $\mat{X}_i$'s such that
    \begin{align*}
        &\; 2 \lVert \vx \rVert_{\left( \mat{X}_i^\top \mat{X}_i \right)^{\dagger}}^2 \log \frac{2km\ell(\ell+1)}{\delta} \\
        \leq&\; 2 \vx^\top \left( \left\lceil \frac{2 g(\pi_{i,\ell})}{\varepsilon_\ell^2} \log \frac{2km\ell (\ell+1)}{\delta} \right\rceil \sum_{\vx' \in \mathcal{A}_{i,\ell}} \pi_\ell(\vx') \vx'^\top \vx' \right)^{-1} \vx \log \frac{2km \ell (\ell+1)}{\delta} \\
        \leq&\; \vepsilon_\ell^2.
    \end{align*}
    Thus, we have
    \begin{align}
        \mathbb{P} \left( \exists \ell > L_c, \; \exists i \in [m], \; \exists \vx \in \mathcal{A}_{i,\ell} :  \langle \hat{\vtheta}_{i,\ell} - \vtheta_i, \vx \rangle \geq \vepsilon_\ell \right) \leq \left( \frac{1}{L_c+1} \right) \delta. \label{eq:collab_rad_fin2}
    \end{align}
    With~\eqref{eq:collab_rad_fin1} and~\eqref{eq:collab_rad_fin2}, we have that
    \begin{align*}
        \mathbb{P} \left( \exists \ell \in \mathbb{N}, \; \exists i \in [m], \; \exists \vx \in \mathcal{A}_{i,\ell} :  \langle \hat{\vtheta}_{i,\ell} - \vtheta_i, \vx \rangle \geq \vepsilon_\ell \right) \leq \delta.
    \end{align*}
\end{proof}

\begin{lemma}
    \label{lem:upp_2}
    Let $\vx_i^* = \max_{\vx \in \mathcal{A}} \langle \vtheta_i, \vx \rangle$. Then, under the good event in Lemma~\ref{lem:upp_1} we have
    \begin{align*}
        \forall \ell \in \mathbb{N} : \vx_i^* \notin \mathcal{A}_{i,\ell}.
    \end{align*}
\end{lemma}
\begin{proof}[Proof of Lemma~\ref{lem:upp_2}]
    For any $\ell \in \mathbb{N}, \; i \in [m]$, suppose that $\vx_i^* \in \mathcal{A}_{i,\ell}$. Then, for any $\vx \in \mathcal{A}_{i,\ell}$,
    \begin{align*}
        \langle \hat{\vtheta}_{i,\ell}, \vx - \vx_i^* \rangle
        =&\; \langle \hat{\vtheta}_{i,\ell} - \vtheta_i, \vx \rangle - \langle \hat{\vmu}_\ell - \vtheta_i, \vx_i^* \rangle - \langle \theta_i, \vx_i^* - \vx \rangle \\
        \leq&\; \langle \hat{\vtheta}_{i,\ell} - \vtheta_i, \vx \rangle - \langle \hat{\vtheta}_{i,\ell} - \vtheta_i, \vx_i^* \rangle \\
        \leq&\; 2 \varepsilon_\ell
    \end{align*}
    from Lemma~\ref{lem:upp_1}.
    Since $\vx_i^* \in \mathcal{A}_1 = \mathcal{A}$, the lemma is proved.
\end{proof}

\begin{lemma}
    \label{lem:upp_3}
    Let $\ell_{i,\vx} = \min \{ \ell \in \mathbb{N} : 4 \varepsilon_\ell < \Delta_{i,\vx} \}$ for any $\vx \in \mathcal{A}$, where $\Delta_{i,\vx} = \langle \vx^*_i - \vx, \vtheta_i \rangle$. Then, under the good event in Lemma~\ref{lem:upp_1} we have
    \begin{align*}
        \forall \vx \in \mathcal{A} : \vx \notin \mathcal{A}_{i,\ell_{i,\vx}+1}. 
    \end{align*}
\end{lemma}
\begin{proof}[Proof of Lemma~\ref{lem:upp_3}]
    The proof is straightforward from Algorithm~\ref{alg:pb}.
\end{proof}

With Lemma~\ref{lem:upp_1}, \ref{lem:upp_2}, and \ref{lem:upp_3}, we can focus on the good event in Lemma~\ref{lem:upp_1} and it's corresponding regret. We have
\begin{align*}
    \hat{R}_{m,n}(\vmu)
    =&\; \sum_{t=1}^n \sum_{i=1}^m \langle \vtheta_i, \vx_i^* \rangle - \langle \vtheta_i, \vx_{i,t} \rangle \\
    =&\; \sum_{i=1}^m \sum_{\vx \in \mathcal{A}} \Delta_{i,\vx} n_{i,\vx} \\
    =&\; \sum_{i=1}^m \left( \sum_{\vx:\Delta_{i,\vx} \leq \eta} \Delta_{i,\vx} n_{i,\vx} + \sum_{\vx:\Delta_{i,\vx} > \eta} \Delta_{i,\vx} n_{i,\vx} \right)
\end{align*}
For any $\vx \in \mathcal{A}$ such that $\Delta_{i,\vx} > \eta$, we have
\begin{align*}
    \ell_{i,\vx}
    \leq \left\lfloor \log \frac{4}{\Delta_{i,\vx}} \right\rfloor + 1
    < \left\lfloor \log \frac{4}{\eta} \right\rfloor + 1.
\end{align*}
Thus,
\begin{align*}
    \hat{R}_{m,n}(\vmu)
    \leq&\; m n \eta + \sum_{i=1}^m \sum_{\vx:\Delta_{i,\vx} > \eta} \sum_{\ell=1}^{\left\lceil \log_2 \frac{4}{\eta} \right\rceil} \Delta_{i,\vx} n_{i,\ell}(\vx) \\
    =&\; mn\eta + \sum_{i=1}^m \sum_{\vx:\Delta_{i,\vx} > \eta} \left( \sum_{\ell=1}^{L'_c} \Delta_{i,\vx} n_{i,\ell}(\vx) + \sum_{\ell=L'_c+1}^{\left\lceil \log_2 \frac{4}{\eta} \right\rceil} \Delta_{i,\vx} n_{i,\ell}(\vx) \right),
\end{align*}
where $L'_c = \min \left\{ L_c, \left\lceil \log \frac{4}{\eta} \right\rceil \right\}$ is the number of collaboration phases.

Suppose that we have $L_c \geq 1$. Let $r_1$ denote the regret in the first collaboration phase. From~\eqref{eq:collab_rad2}, we know that
\begin{align*}
    \left|\langle \vmu - \vtheta_i, \vx \rangle \right|
    \leq \sqrt{ 2 \left( \max_{\vx \in \mathcal{A}} \lVert \vx \rVert_\mat{C}^2 \right) \log \frac{4km}{\delta} }.
\end{align*}
Thus, in the first phase we have
\begin{align*}
    &\; \sum_{i=1}^m \sum_{\vx:\Delta_{i,\vx} > \eta} \Delta_{i,\vx} n_{i,1}(\vx) \\
    \leq&\; \sum_{i=1}^m \sum_{\vx:\Delta_{i,\vx} > \eta} \left( 2 \max_{\vx \in \mathcal{A}} \langle \vmu, \vx \rangle + 2 \sqrt{ 2 \left( \max_{\vx \in \mathcal{A}} \lVert \vx \rVert_\mat{C}^2 \right) \log \frac{4km}{\delta} } \, \right) \left\lceil \frac{8 \pi_\ell(\vx) g(\pi_\ell)}{m (1/2)^2} \log \frac{2k\ell(\ell+1)}{\delta} \right\rceil \\
    \leq&\; \sum_{i=1}^m \sum_{\vx:\Delta_{i,\vx} > \eta} \left( 2\max_{\vx \in \mathcal{A}} \langle \vmu, \vx \rangle + 2 \sqrt{ 2 \left( \max_{\vx \in \mathcal{A}} \lVert \vx \rVert_\mat{C}^2 \right) \log \frac{4km}{\delta} } \, \right) \left( \frac{32 \pi_\ell(\vx) d}{m} \log \frac{2k\ell(\ell+1)}{\delta} + 1 \right) \\
    \leq&\; \left( 2 \max_{\vx \in \mathcal{A}} \langle \vmu, \vx \rangle + 2 \sqrt{ 2 \left( \max_{\vx \in \mathcal{A}} \lVert \vx \rVert_\mat{C}^2 \right) \log \frac{4km}{\delta} } \, \right) \left( m d(d+1) + 32 d \log \frac{2k\ell(\ell+1)}{\delta} \right).
\end{align*}
Let us denote this by 
\begin{align*}
    r_1
    =&\; \left( 2 \max_{\vx \in \mathcal{A}} \langle \vmu, \vx \rangle + 2 \sqrt{ 2 \left( \max_{\vx \in \mathcal{A}} \lVert \vx \rVert_\mat{C}^2 \right) \log \frac{4km}{\delta} } \, \right) \left( m d(d+1) + 32 d \log \frac{2k\ell(\ell+1)}{\delta} \right) \\
    =&\; O \left( md^2 \right).
\end{align*}
Notice that $r_1$ does not depend on $n$. For the rest of the collaboration phases, we have
\begin{align*}
    \sum_{i=1}^m \sum_{\vx:\Delta_{i,\vx} > \eta} \sum_{\ell=2}^{L'_c} \Delta_{i,\vx} n_{i,\ell}(\vx)
    \leq&\; \sum_{i=1}^m \sum_{\ell=2}^{L'_c} \sum_{\vx:\Delta_{i,\vx} > \eta} 2 \varepsilon_\ell \left\lceil \frac{8 \pi_\ell(\vx) g(\pi_\ell)}{m \varepsilon_\ell^2} \log \frac{2k\ell(\ell+1)}{\delta} \right\rceil \\
    \leq&\; m \sum_{\ell=2}^{L'_c} 2 \varepsilon_\ell \left( \left\lvert \mathrm{Supp}(\pi_\ell) \right\rvert + \frac{8 d}{m \varepsilon_\ell^2} \log \frac{2k\ell(\ell+1)}{\delta} \right) \\
    \leq&\; m \sum_{\ell=2}^{L'_c} 2 \cdot 2^{-\ell} \left( \frac{d(d+1)}{2} + \frac{8d}{m 2^{-2\ell}} \log \frac{2k\ell(\ell+1)}{\delta} \right) \\
    \leq&\; md(d+1) \left( \frac{1}{2} - 2^{-L'_c-1} \right) + 16d \left( \sum_{\ell=2}^{L'_c} 2^\ell \log \frac{2k\ell(\ell+1)}{\delta} \right) \\
    \leq&\; md(d+1) \left( \frac{1}{2} - 2^{-L'_c-1} \right) + 2^{L'_c+1} \left( 16d \right) \log \frac{4k {L'_c}^2}{\delta}.
\end{align*}
For the local training phase, we have
\begin{align*}
    \sum_{i=1}^m \sum_{\vx:\Delta_{i,\vx} > \eta} \sum_{\ell=L'_c+1}^{\left\lceil \log_2 \frac{4}{\eta} \right\rceil} \Delta_{i,\vx} n_{i,\ell}(\vx)
    \leq&\; \sum_{i=1}^m \sum_{\ell=L'_c+1}^{\left\lceil \log_2 \frac{4}{\eta} \right\rceil} \sum_{\vx:\Delta_{i,\vx} > \eta} 2 \varepsilon_\ell \left\lceil \frac{2 \pi_{i,\ell}(\vx) g_i(\pi_{i,\ell})}{\varepsilon_\ell^2} \log \frac{2km\ell(\ell+1)}{\delta} \right\rceil \\
    \leq&\; \sum_{i=1}^m \sum_{\ell=L'_c+1}^{\left\lceil \log_2 \frac{4}{\eta} \right\rceil} 2 \varepsilon_\ell \left( \left\lvert \mathrm{Supp}(\pi_{i,\ell}) \right\rvert + \frac{2 d}{\varepsilon_\ell^2} \log \frac{2km\ell(\ell+1)}{\delta} \right) \\
    \leq&\; m \sum_{\ell=L'_c+1}^{\left\lceil \log_2 \frac{4}{\eta} \right\rceil} 2 \cdot 2^{-\ell} \left( \frac{d(d+1)}{2} + \frac{2d}{2^{-2\ell}} \log \frac{2km\ell(\ell+1)}{\delta} \right) \\
    \leq&\; md(d+1) 2^{-L'_c-1} + 4dm \left( \sum_{\ell=L'_c+1}^{\left\lceil \log_2 \frac{4}{\eta} \right\rceil} 2^\ell \log \frac{2km\ell(\ell+1)}{\delta} \right) \\
    \leq&\; md(d+1) 2^{-L'_c-1} + \left( 2^{\left\lceil \log_2 \frac{4}{\eta} \right\rceil + 1} - 2^{L'_c+1}\right) 4dm \log \frac{4km \left( L'_c \right)^2}{\delta}.
\end{align*}
Thus, the over all regret
\begin{align*}
    \hat{R}_{m,n}(\vmu)
    \leq&\; r_1 + mn\eta + \frac{1}{2} md(d+1) + 2^{L'_c+1} \left( 16d \log \frac{4k (L'_c)^2}{\delta} \right) \\
    &\; + \left( 2^{\left\lceil \log_2 \frac{4}{\eta} \right\rceil + 1} - 2^{L'_c+1}\right) \left(  4dm \log \frac{4km \left( L'_c \right)^2}{\delta} \right) \\
    \leq&\; r_1 + mn\eta + \frac{1}{2} md(d+1) + \min \left\{ \frac{1}{h_{\sigma, \delta}}, \frac{16}{\eta} \right\} \left( 16d \log \frac{4k \left( \log_2 \frac{8}{\eta} \right)^2}{\delta} \right) \\
    &\; + \left( \frac{16}{\eta} - \min \left\{ \frac{1}{h_{\sigma, \delta}}, \frac{16}{\eta} \right\} \right) \left(  4dm \log \frac{4km \left( \log_2 \frac{8}{\eta} \right)^2}{\delta} \right)
\end{align*}
for any $\eta > 0$.

When $L_c = 0$, it means that $h_{\sigma, \delta}$ is large and there is no collaboration phase. Thus, in the first local training phase, we have
\begin{align*}
    &\; \sum_{i=1}^m \sum_{\vx:\Delta_{i,\vx} > \eta} \Delta_{i,\vx} n_{i,1}(\vx) \\
    \leq&\; \sum_{i=1}^m \sum_{\vx:\Delta_{i,\vx} > \eta} \left( 2 \max_{\vx \in \mathcal{A}} \langle \vmu, \vx \rangle + 2 \sqrt{ 2 \left( \max_{\vx \in \mathcal{A}} \lVert \vx \rVert_\mat{C}^2 \right) \log \frac{4km}{\delta} } \, \right) \left\lceil \frac{8 \pi_{i,\ell}(\vx) g(\pi_{i,\ell})}{m (1/2)^2} \log \frac{2k\ell(\ell+1)}{\delta} \right\rceil \\
    \leq&\; \sum_{i=1}^m \sum_{\vx:\Delta_{i,\vx} > \eta} \left( 2\max_{\vx \in \mathcal{A}} \langle \vmu, \vx \rangle + 2 \sqrt{ 2 \left( \max_{\vx \in \mathcal{A}} \lVert \vx \rVert_\mat{C}^2 \right) \log \frac{4km}{\delta} } \, \right) \left( \frac{32 \pi_{i,\ell}(\vx) d}{m} \log \frac{2k\ell(\ell+1)}{\delta} + 1 \right) \\
    \leq&\; m d(d+1) + \left( 2 \max_{\vx \in \mathcal{A}} \langle \vmu, \vx \rangle + 2 \sqrt{ 2 \left( \max_{\vx \in \mathcal{A}} \lVert \vx \rVert_\mat{C}^2 \right) \log \frac{4km}{\delta} } \, \right) \left( 32 d \log \frac{2k\ell(\ell+1)}{\delta} \right) \\
    =&\; r_1
\end{align*}
where we defined $r_1$ earlier for the case when $L_c \geq 1$. Thus, we can see that the same upper bound on the regret holds for this case. Now, for the over all regret bound
\begin{align}
    \hat{R}_{m,n}(\vmu)
    \leq&\; r_1 + mn\eta + \frac{1}{2} md(d+1) + \min \left\{ \frac{1}{h_{\sigma, \delta}}, \frac{16}{\eta} \right\} \left( 16d \log \frac{4k \left( \log_2 \frac{8}{\eta} \right)^2}{\delta} \right) \nonumber \\
    &\; + \left( \frac{16}{\eta} - \min \left\{ \frac{1}{h_{\sigma, \delta}}, \frac{16}{\eta} \right\} \right) \left(  4dm \log \frac{4km \left( \log_2 \frac{8}{\eta} \right)^2}{\delta} \right) \label{eq:upper_opt}
\end{align}
for any $\eta > 0$, we look into different cases. Instead of $h_{\sigma, \delta}$, we look into
\begin{align*}
    \tilde{h}_{\sigma, \delta}
    = \sqrt{\frac{n}{d}} h_{\sigma, \delta}
    = \left( \sqrt{ \frac{2n}{dm} \log \frac{2k}{\delta} } + \sqrt{ \frac{2n}{d} \log \frac{2k}{\delta} } \right) \sqrt{ \max_{\vx \in \mathcal{A}} \vx^\top \mat{C} \vx }
\end{align*}
First, if $\tilde{h}_{\sigma, \delta} = O \left( m^{-\frac{1}{2}} \right)$, say, $\tilde{h}_{\sigma, \delta} \leq c_1 m^{-1/2}$ for some constant $c_1 > 0$. Then, by choosing
\begin{align}
    \eta = 16c_1 \sqrt{\frac{d}{mn}} \geq 16 \sqrt{\frac{d}{n}} \tilde{h}_{\sigma, \delta} = 16 h_{\sigma, \delta},
\end{align}
we have
\begin{align*}
    \hat{R}_{m,n}(\vmu)
    \leq&\; r_1 + mn\eta + \frac{1}{2} d(d+1) + \frac{16}{\eta} \left( 16d \log \frac{4k \left( \log_2 \frac{8}{\eta} \right)^2}{\delta} \right) \\
    \leq&\; r_1 + 16 c_1 \sqrt{dmn} + \frac{1}{2} d(d+1) + \frac{16}{c_1} \sqrt{dmn} \log \left( \frac{4k \left( \log_2 \sqrt{mn / 4 c_1^2 d} \right)^2}{\delta} \right) \\
    =&\; \tilde{O} \left( \sqrt{dmn} + m d^2 \right)
\end{align*}

Secondly, if we have $\tilde{h}_{\sigma, \delta} = \Theta(m^{-\gamma})$ for some $\gamma \in [0, -\frac{1}{2}]$, say, 
\begin{align*}
    c_2 m^{-\gamma} \leq \tilde{h}_{\sigma, \delta} \leq c_3 m^{-\gamma},
\end{align*}
for some constant $c_2, c_3  > 0$.
Then, by choosing
\begin{align}
    \eta = 16 \tilde{h}_{\sigma, \delta} \sqrt{\frac{d}{n}} = 16 h_{\sigma, \delta},
\end{align}
we have
\begin{align*}
    \hat{R}_{m,n}(\vmu)
    \leq&\; r_1 + c_3 m^{1-\gamma} \sqrt{dn} + \frac{1}{2} d(d+1) + \frac{1}{c_2 m^{-\gamma}} \sqrt{\frac{n}{d}} \left( 16d \log \frac{4k \left( \log_2 m^{\gamma} \sqrt{n/4c_2^2 d} \right)^2}{\delta} \right) \\
    \leq&\; r_1 + c_3 m^{1-\gamma} \sqrt{dn} + \frac{1}{2} d(d+1) + \frac{16}{c_2} m^{\gamma} \sqrt{dn} \log \frac{4k \left( \log_2 m^{\gamma} \sqrt{n/4c_2^2 d} \right)^2}{\delta} \\
    =&\; \tilde{O} \left( m^{1-\gamma} \sqrt{dn} + md^2 \right).
\end{align*}

Finally, if we have $\tilde{h}_{\sigma, \delta} = \Omega(1)$, by choosing $\eta = \min \left\{16 \tilde{h}_{\sigma, \delta} \sqrt{d/n}, \sqrt{d/n} \right\}$, we can ontain
\begin{align*}
    \hat{R}_{m,n}(\vmu)
    =\tilde{O} \left( m \sqrt{dn} + md^2 \right).
\end{align*}

\subsection{Proof of Theorem~\ref{thm:alg_infinite}}
\label{subsec:proof_infinite}
We apply the typical $\varepsilon$-net argument to cover the unit-ball action set with a finite action set. By Corollary 4.2.13 in~\cite{vershynin2018high}, the covering number of the $d$-dimensional unit Euclidean ball satisfies
\begin{align*}
    \mathcal{N}(\mathcal{B}_d, \varepsilon) \leq \left( 1 + \frac{2}{\varepsilon} \right)^d.
\end{align*}
Let $\mathcal{A}_c$ be the covering and use $\mathcal{A}_c$ as the action set for Algorithm~\ref{alg:pb}. Plug
\begin{align*}
    k = \lvert \mathcal{A}_c \rvert \leq \left( 1 + \frac{2}{\varepsilon} \right)^d
\end{align*}
into~\eqref{eq:upper_opt}, we have
\begin{align}
    \hat{R}_{m,n}(\vmu)
    \leq&\; r_1 + mn\eta + \frac{1}{2} md(d+1) + \min \left\{ \frac{1}{h_{\sigma, \delta}}, \frac{16}{\eta} \right\} \left( 16d^2 \log \frac{4 \left( 1 + \frac{2}{\varepsilon} \right) \left( \log_2 \frac{8}{\eta} \right)^2}{\delta} \right) \nonumber \\
    &\; + \left( \frac{16}{\eta} - \min \left\{ \frac{1}{h_{\sigma, \delta}}, \frac{16}{\eta} \right\} \right) \left(  4d^2 m \log \frac{4 \left( 1 + \frac{2}{\varepsilon} \right) m \left( \log_2 \frac{8}{\eta} \right)^2}{\delta} \right) + \varepsilon m n \nonumber
\end{align}
for any $\varepsilon > 0$. Thus, if we choose $\varepsilon = O \left( \frac{1}{\sqrt{mn}} \right)$ and optimize over $\eta$ like we did for Theorem~\ref{thm:alg_finite}, we will get the results in Theorem~\ref{thm:alg_infinite} with optimal $\eta$ shifted by the order of $\sqrt{d}$ comparing to the proof of Theorem~\ref{thm:alg_finite}.

\subsection{The additional term in the upper bound}
\label{app:upper_md2}
The extra $md^2$ term in the regret upper bound comes from~\eqref{eq:alg_n_col} in our collaborative phased elimination algorithm.
As we will see in the details of Algorithm~\ref{alg:pb} in Section~\ref{sec:algorithm_analysis}, in each phase of the collaborative learning stage, we first decide how many times each action $\vx$ will be selected to reach the target error in the phase. Let us denote it by $n_{\ell}(\vx)$ for now. Then, we split the workload so that each user is responsible to select each action $\vx$ for $n_{i,\ell}(\vx) = \left\lceil n_{\ell}(\vx) / m \right\rceil$ times. Since the ceiling function can increase the value by almost 1 for each user and action in the worst-case, and the number of users is $m$ and the number of the actions to be pulled (the cardinality of support of the design) can be upper bounded by $d(d+1)/2$ (\cite{lattimore2020bandit}, Theorem~21.1), we will have at most extra $md^2$ pulls in each phase. By the way we choose the exponentially decreasing target error, this leads to an $md^2$ term in the regret. A more efficient way to split the workload might be first calculating the total rounds needed for all action combined, $n_\ell = \sum_{\vx} n_{\ell}(\vx)$, then split it among the users. In this case, the users might be selecting different actions in the phase, and our proof need to be modified since we assume that each user selects the same actions during the collaborative learning stage. After the change, we will have an extra $(m + d^2)$ pull instead. The $m$ term remains due to the fact that our problem formulation is synchronized. For instance, if the number of users $m$ is much larger than the workload $n_\ell$ in some phase $\ell$, our current formulation does not allow users to remain idle while a each of the $n_\ell$ selected user perform a pull. And thus since all users are forced to select some action in the phase, there are at least $m$ pulls instead of $n_\ell$ pulls in each stage. This cause the $m$ term in the regret. To resolve this, we have to remove the synchronized assumption on the formulation. For example, if the users are sequentially selecting actions, there can be only a part of the users that have to select actions in every phase. This prevent from the extra $m$ term in the regret. However, the regret analysis will be different from our current one.

\subsection{Sub-Gaussian population distributions}
\label{app:subgaussian}

Our results on the regret analysis for CP-PE can be translated to the case where $(\vtheta_i-\vmu)$'s are sub-Gaussian random vectors. This is because sub-Gaussian random variables have similar concentration results to Gaussian random variables. Let us formally define the term first:

\begin{definition}[Sub-Gaussian random variables]
    A random variable $\pmb{X}$ is sub-Gaussian with variance proxy $\sigma^2$ if
    \begin{align*}
        \mathbb{E} \left[ e^{\lambda \pmb{X}} \right] \leq e^{\frac{\lambda^2 \sigma^2}{2}}
    \end{align*}
    for all $\lambda \in \mathbb{R}$.
\end{definition}

With the definition of a sub-Gaussian random variable, we can extend the idea to a random vector.

\begin{definition}[Sub-Gaussian random vectors]
    A random vector $\pmb{X}$ in the $d$-dimension space is said to be sub-Gaussian with variance proxy $\sigma^2$ if it is centered and for any $\vu \in \mathbb{R}^d$ such that $\lVert \vu \rVert_2 = 1$, the real random variable $\vu^\top \pmb{X}$ is sub-Gaussian with variance proxy $\sigma^2$. We write $\pmb{X} \sim \mathrm{subG}(\sigma^2)$.
\end{definition}

Let $\vu_i = \vtheta_i-\vmu$ follow a sub-Gaussian distribution with variance proxy $\sigma^2$ for each $i \in [m]$. Then, for any $\vx \in \mathcal{B}_d^2$, $\vu_i^\top \vx$ is $\mathrm{subG}(\sigma^2)$ and thus
\begin{align}
    \left\langle \frac{1}{m} \sum_{i=1}^m \vu_i, \vx_i \right\rangle \sim \mathrm{subG}( \sigma^2 / m).
\end{align}
From the concentration property of sub-Gaussian random variables (\cite{wainwright2019high}, Chapter~2.1), we have
\begin{align}
    \mathbb{P} \left\{ \left\langle \frac{1}{m} \sum_{i=1}^m \vu_i, \vx_i \right\rangle \geq s \right\}
    \leq \exp \left\{ - \frac{m s^2}{2 \sigma^2} \right\}. \label{eq:upper_proof_sub_gaus}
\end{align}
Recall that in Section~\ref{subsec:proof_finite}, the regret analysis of CP-PE, the Gaussianity of $\vtheta_i$ is only used for the additivity of Gaussianity and its concentration properties. Since similar properties are also established for sub-Gaussian random variables in~\eqref{eq:upper_proof_sub_gaus}, we can simply replace~\eqref{eq:upper_proof_gaus} with sub-Gaussian $\vu_i$'s, the replace~\eqref{eq:upper_proof_gaus_concen} and~\eqref{eq:upper_proof_gaus_concen2} with the concentration inequality of sub-Gaussian random variables. The rest of the proof is essentially the same. Additionally, CP-PE is not using the fact that $\vtheta_i$ follows a Gaussian distribution. Thus, our algorithm and the regret analysis also work for any sub-Gaussian population distribution.

%% file: Sections/experiments.tex
\section{Experiments}
\label{sec:experiments}

In our experiments, we compare {CP-PE} with Federated Phased Elimination ({Fed-PE}, all agents perform phased elimination as if they have the same unknown parameter) and Individual Phased Elimination ({Ind-PE}, each agent performs phased elimination on it own and evaluate the joint regret). We pick the population parameter $\sigma$ such that the heterogeneity falls in the second regime, and {CP-PE} usually performs one to two phases in the collaborative stage depending on the instance of $\{\vtheta_i\}_{i=1:m}$ in each run.

\subsection{Experiment setting}
\label{app:exp_setting}
We compare {CP-PE}, {Fed-PE}, and {Ind-PE} on synthetic datasets. In the first experiment, we set the number of users $m=100$, number of rounds $n=15000$, and dimension $d=2$. We set $\vmu = [1 \; 0]^\top$, $\sigma = 0.3$. We let the action set $\mathcal{A}$ consist of 10 actions ($k=10$) uniformly distributed on the unit ball. We set the variance of noise to be $\sigma_0^2 = 1$ and $\delta = 0.01$. In each run, we randomly sample the unknown parameter $\vtheta_i \sim \mathcal{N}(\vmu, \sigma^2 \mat{I})$ for $i \in [m]$. Then, we apply all {CP-PE}, {Fed-PE}, and {Ind-PE} on the instance of $\{\vtheta_i\}_{i=1:m}$. We evaluate the joint cumulative regret for each algorithm and normalized it by the number of users $m$. Our result is averaged over 80 independent runs and the 1-sigma error bar is plotted in the figure.

\subsection{Experiment results}

\begin{figure*}[t!]
    \centering
    \begin{subfigure}[t]{0.45\textwidth}
        \centering
        \includegraphics[width=1\linewidth, trim={0 0 0 0},clip]{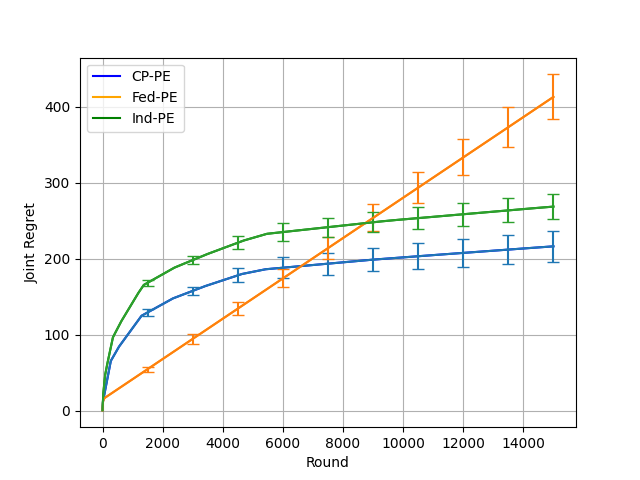}
        \caption{Joint cumulative regret obtained by CP-PE, Fed-PE, and Ind-PE. In the experiment, we have $m=100$, $n=15000$, $d=2$, $\sigma=0.3$.}
        \label{fig:m100}
    \end{subfigure}%
    ~ 
    \begin{subfigure}[t]{0.45\textwidth}
        \centering
        \includegraphics[width=1\linewidth, trim={0 0 0 0},clip]{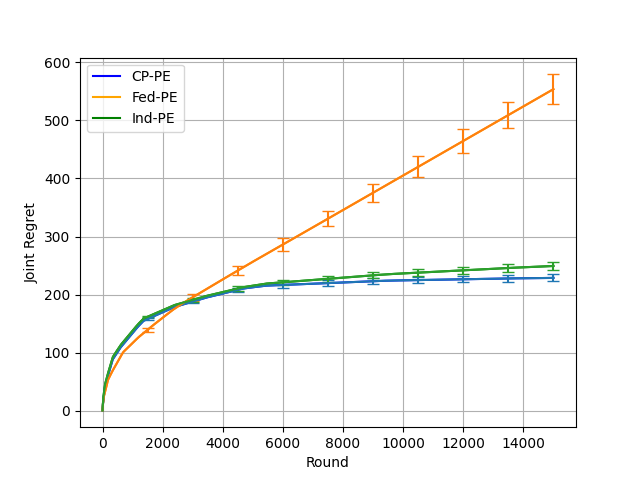}
        \caption{Joint cumulative regret obtained by CP-PE, Fed-PE, and Ind-PE. In the experiment, we have $m=2$, $n=15000$, $d=2$, $\sigma=0.3$.}
        \label{fig:m2}
    \end{subfigure}
    \caption{Comparing regret obtained by CP-PE, Fed-PE, and Ind-PE.}
\end{figure*}

In the first experiment, there are 100 users ($m=100$) and the level of heterogeneity $\sigma=0.3$. Figure~\ref{fig:m100} shows the joint cumulative regret for each method. We can see CP-PE has the best final regret. Although {Fed-PE} has the best performance in the beginning, the cumulative regret grows linearly in $n$ and is outperformed by {CP-PE} and {Ind-PE} eventually. The head start of {Fed-PE} is due to the benefit of collaboration, which significantly increases the effective sample size and eliminates the bad actions rapidly. However, without personalization, agents are forced to find a global best action that can be distinct from each agent's individual best action, and thus a linear regret is inevitable. On the other hand, {CP-PE} benefits from the collaborative learning stage and switch to the personalized learning stage before each agent's best action is eliminated. The transition happens in a few of the early rounds since the first few phases in phased elimination are composed of much fewer pulls of actions compared to the later phases, and the pulls are even equally divided among each user in the collaborative learning stage. By monitoring each run, we can see that agents under CP-PE rapidly shrink the action set in the collaborative learning stage and start with a much smaller action set when starting the personalized learning stage. Thus, CP-PE obtains a better regret than {Ind-PE}.

In another experiment, we decrease the number of users to $m=2$ and the result is averaged over 350 independent runs. The joint cumulative regrets obtained by CP-PE, Fed-PE, and Ind-PE are shown in Figure~\ref{fig:m2}. We can see that the plot looks similar to Figure~\ref{fig:m100} since we also consider the regime where CP-PE usually perform one to two phases in the collaborative learning stage and switch to the personalized learning stage. The difference shown in the plot is that since the number of users $m$ decreases, the performance gap between CP-PE and Ind-PE becomes smaller. Comparing the result to Figure~\ref{fig:m100}, we can see the benefit of collaboration scale with the number of users.

%% file: Sections/conclusion.tex
\section{Conclusion}

We provide the first complete characterization of the benefit of collaboration under our hierarchical/empirical Bayesian framework for heterogeneous linear bandits. Overcoming technical challenges, we establish a minimax lower bound and propose CP-PE, an algorithm that achieves minimax-optimal regret up to polylogarithmic factors. Despite this complete characterization, several important directions remain open for future work. First, one can consider alternative heterogeneity models beyond the Gaussian/sub-Gaussian population studied in this work. Second, our proposed algorithm CP-PE depends on the known level of heterogeneity.
How agents can collaborate without this knowledge, and whether there is a fundamental gap in regret from the known heterogeneity case, remains an open and important question of ongoing investigations.